\documentclass{article}

\usepackage{microtype}
\usepackage{graphicx}
\usepackage{subcaption}
\usepackage{booktabs} %
\usepackage{multirow}
\usepackage[inline]{enumitem}

\clearpage{}%
\usepackage{amsthm}
\usepackage{xspace}
\usepackage{tikz}
\usetikzlibrary{arrows}
\usetikzlibrary{positioning}

\definecolor{ForestGreen}{RGB}{34,139,34}

\newcommand{\homs}{\Hom}
\newcommand{\Hom}{\ensuremath{\mathsf{Hom}}\xspace}

\newcommand{\subs}{\Sub}
\newcommand{\Sub}{\ensuremath{\mathsf{Sub}}\xspace}

\newcommand{\indsubs}{\ensuremath{\mathsf{IndSub}}\xspace}

\newcommand{\injs}{\ensuremath{\mathsf{Inj}}\xspace}

\newcommand{\anchor}{\ensuremath{\circledast}}

\newcommand{\spasm}{\ensuremath{\mathsf{Spasm}}\xspace}
\newcommand{\homdetf}{\ensuremath{\mathsf{homdet}}\xspace}
\newcommand{\detf}{\ensuremath{\mathsf{det}}\xspace}

\newcommand{\isomorphic}{\simeq}

\newcommand{\gnn}{\ensuremath{\Psi}}
\newcommand{\gnnmodel}{\ensuremath{\Psi_\theta}}
\newcommand{\bfI}{\mathbf{I}}

\newcommand{\extraparams}{\mathbb{A}}
\newcommand{\extrabasis}{\mathbb{B}}

\newcommand{\supp}{\mathsf{Supp}}
\newcommand{\graphs}{\ensuremath{\Omega}}

\newcommand{\nop}[1]{}

\newcommand{\update}{\small \textsc{Upd}\xspace}
\newcommand{\aggregate}{\small\textsc{Agg}\xspace}

\newcommand{\claimqed}{\ensuremath{\triangleleft}}

\theoremstyle{remark}
\newtheorem{claim}{Claim}

\newcommand\xqed[1]{%
  \leavevmode\unskip\penalty9999 \hbox{}\nobreak\hfill
  \quad\hbox{#1}}
\newcommand\exend{\xqed{$\triangle$}}

\newcommand{\vartxt}[1]{\tiny $\pm{}#1$}
\newcommand{\vartxtm}[1]{\scriptsize $\pm{}#1$}

\newcommand{\subprob}{\#\textsc{Sub}\xspace}

\newcommand{\psubprob}{\textsc{p-\#Sub}\xspace}

\usepackage{tabularx, environ, caption}
\makeatletter
\newcolumntype{\expand}{}
\long\@namedef{NC@rewrite@\string\expand}{\expandafter\NC@find}

\NewEnviron{problem}[2][]{%
  \def\problem@arg{#1}%
  \def\problem@framed{framed}%
  \def\problem@lined{lined}%
  \def\problem@doublelined{doublelined}%
  \ifx\problem@arg\@empty%
    \def\problem@hline{}%
  \else%
    \ifx\problem@arg\problem@doublelined%
      \def\problem@hline{\hline\hline}%
    \else%
      \def\problem@hline{\hline}%
    \fi%
  \fi%
  \ifx\problem@arg\problem@framed%
    \def\problem@tablelayout{|>{\itshape}lX|c}%
    \def\problem@title{\multicolumn{2}{|l|}{%
        \raisebox{-\fboxsep}{\textsc{#2}}%
      }}%
  \else
    \def\problem@tablelayout{>{\itshape}lXc}%
    \def\problem@title{\multicolumn{2}{l}{%
        \raisebox{-\fboxsep}{\textsc{\large #2}}%
      }}%
  \fi%
  \smallskip\par\noindent%
  \begin{center}
  \begin{tabularx}{0.9\columnwidth}{\expand\problem@tablelayout}%
    \problem@hline%
    \problem@title\\[2\fboxsep]%
    \BODY\\\problem@hline%
  \end{tabularx}%
\end{center}

  \medskip\par%
}
\makeatother
\clearpage{}%
\clearpage{}%
\usepackage{tkz-graph}
\usepackage{adjustbox}

\newcommand{\graphpreamble}{%
  \SetVertexNoLabel
  \SetVertexMath
  \SetVertexSimple[MinSize = 1pt,
  LineWidth = 0pt,
  LineColor = black,%
  FillColor = black]
  \renewcommand*{\VertexInnerSep}{1pt}

}

\newcommand{\drawStarthree}{%
  \begin{tikzpicture}[scale=0.14, baseline={(0,-0.06)}]%
    \GraphInit[vstyle=Classic]%
    \graphpreamble
    \SetGraphUnit{1}
    \Vertex[x=0, y=0]{A}
    \Vertex[x=1.2, y=-1]{B}
    \Vertex[x=0.0, y=1.3]{C}
    \Vertex[x=-1.2, y=-1]{D}
    \Edges(A,B)%
    \Edges(A,C)%
    \Edges(A,D)%
  \end{tikzpicture}
}
\newcommand{\drawCfour}{%
  \adjustbox{}{
  \begin{tikzpicture}[scale=0.24,rotate=45, baseline={(0,-0.0)}]%
    \GraphInit[vstyle=Classic]%
    \graphpreamble
    \SetGraphUnit{1}
    \Vertices{square}{A,B,C,D}
    \Edges(A,B,C,D,A)%
  \end{tikzpicture}
  }
}

\newcommand{\drawKthree}{%
\adjustbox{}{
  \begin{tikzpicture}[scale=0.23,rotate=180,baseline={(0,-0.45)}]%
    \GraphInit[vstyle=Classic]%
    \graphpreamble
    \SetGraphUnit{1}
    \Vertex[x=0, y=2]{A}
    \Vertex[x=1.5, y=2]{B}
    \Vertex[x=0.75, y=1]{C}
    \Edges(A,B,C,A)%
  \end{tikzpicture}
  }
}

\newcommand{\drawPtwo}{%
  \begin{tikzpicture}[scale=0.23,rotate=30,baseline={(0,0.02)}]%
    \GraphInit[vstyle=Classic]%
    \graphpreamble
    \SetGraphUnit{1}
    \Vertices{line}{A,B,C}
    \Edges(A,B,C)%
  \end{tikzpicture}
}

\newcommand{\drawPone}{%
  \begin{tikzpicture}[scale=0.23,rotate=30]%
    \GraphInit[vstyle=Classic]%
    \graphpreamble
    \SetGraphUnit{1}
    \Vertices{line}{A,B}
    \Edges(A,B)%
  \end{tikzpicture}
}

\newcommand{\drawKthreeplusone}{%
  \begin{tikzpicture}[scale=0.16,rotate=130,baseline={(0,-0.17)}]%
    \GraphInit[vstyle=Classic]%
    \graphpreamble
    \SetGraphUnit{1}
    \Vertex[x=0, y=2]{A}
    \Vertex[x=1.5, y=2]{B}
    \Vertex[x=0.75, y=1]{C}
    \Vertex[x=0.75, y=0]{D}
    \Edges(A,B,C,A)%
    \Edges(C,D)
  \end{tikzpicture}
}

\newcommand{\drawCfive}{%
  \begin{tikzpicture}[scale=0.18,rotate=17.5, baseline={(0,-0.1)}]%
    \GraphInit[vstyle=Classic]%
    \graphpreamble
    \SetGraphUnit{1}
    \Vertices{circle}{A,B,C,D,E}
    \Edges(A,B,C,D,E,A)%
  \end{tikzpicture}
}

\newcommand{\introexscale}{0.4}

\newcommand{\drawCfiveBigZ}{%
  \begin{tikzpicture}[scale=\introexscale,rotate=17.5, baseline={(0,0)}]%
    \GraphInit[vstyle=Classic]%
    \graphpreamble
    \SetGraphUnit{1}
    \Vertices{circle}{A,B,C,D,E}
    \Edges(A,B,C,D,E,A)%

  \end{tikzpicture}
}

\newcommand{\drawCfiveBigA}{%
  \begin{tikzpicture}[scale=\introexscale,rotate=17.5, baseline={(0,0)}]%
    \GraphInit[vstyle=Classic]%
    \graphpreamble
    \SetGraphUnit{1}
    \Vertices{circle}{A,B,C,D,E}
    \Edges(A,B,C,D,E,A)%

    \draw[dashed] (A) -- (C);
  \end{tikzpicture}
}

\newcommand{\drawCfiveBigB}{%
  \begin{tikzpicture}[scale=\introexscale,rotate=17.5, baseline={(0,0)}]%
    \GraphInit[vstyle=Classic]%
    \graphpreamble
    \SetGraphUnit{1}
    \Vertices{circle}{A,B,C,D,E}
    \Edges(A,B,C,D,E,A)%

   \draw[dashed] (A) -- (C);
    \draw[dashed] (B) -- (D);
  \end{tikzpicture}
}

\newcommand{\drawKthreeplusoneBigA}{%
  \begin{tikzpicture}[scale=\introexscale,rotate=0,baseline={(0,0.4)}]%
    \GraphInit[vstyle=Classic]%
    \graphpreamble
    \SetGraphUnit{1}
    \Vertex[x=0, y=0]{A}
    \Vertex[x=1.5, y=0]{B}
    \Vertex[x=0.75, y=1]{C}
    \Vertex[x=0.75, y=2]{D}
    \Edges(A,B,C,A)%
    \Edges(C,D)
  \end{tikzpicture}
}

\newcommand{\drawKthreeplusoneBigB}{%
  \begin{tikzpicture}[scale=\introexscale,rotate=0,baseline={(0,0.4)}]%
    \GraphInit[vstyle=Classic]%
    \graphpreamble
    \SetGraphUnit{1}
    \Vertex[x=0, y=0]{A}
    \Vertex[x=1.5, y=0]{B}
    \Vertex[x=0.75, y=1]{C}
    \Vertex[x=0.75, y=2]{D}
    \Edges(A,B,C,A)%
    \Edges(C,D)

    \draw[dashed] (D) to[out=190,in=150] (A);
  \end{tikzpicture}
}

\newcommand{\drawKthreeBigB}{%
  \begin{tikzpicture}[scale=\introexscale,rotate=0,baseline={(0,0.4)}]%
    \GraphInit[vstyle=Classic]%
    \graphpreamble
    \SetGraphUnit{1}
    \Vertex[x=0, y=0]{A}
    \Vertex[x=1.5, y=0]{B}
    \Vertex[x=0.75, y=1]{C}

    \Edges(A,B,C,A)%
  \end{tikzpicture}
}

\newcommand{\drawLthreeonetwist}{%
  \begin{tikzpicture}[scale=0.8,rotate=0, baseline={(0,0)}]%
    \GraphInit[vstyle=Classic]%
    \graphpreamble
    \SetGraphUnit{1}
    \Vertices{circle}{A,B,C,D,E,F,G,H}
    \Vertex{M}
    \Edges(A,B,C,D,E,F,G,H,A)%
    \Edges(B,H)
    \Edges(C,M,G)
    \Edges(D,F)
  \end{tikzpicture}
}

\newcommand{\drawLthreeoneCFI}{%
  \begin{tikzpicture}[scale=0.8,rotate=0, baseline={(0,0)}]%
    \GraphInit[vstyle=Classic]%
    \graphpreamble
    \SetGraphUnit{1}
    \Vertices{circle}{A,B,C,D,E,F,G}
    \Vertex[x=0,y=0.3]{M}
    \Vertex[x=0,y=-0.3]{N}
    \Edges(A,B,C,D,E,F,G)%
    \Edges(G,M,B,M,C)
    \Edges(A,N,F,N,E)
  \end{tikzpicture}
}

\newcommand{\drawCfiveTwist}{%
  \begin{tikzpicture}[scale=0.8,rotate=0, baseline={(0,0)}]%
    \GraphInit[vstyle=Classic]%
    \graphpreamble
    \SetGraphUnit{1}
    \Vertices{circle}{A,B,C,D,E,F,G,H,I,J}
    \Edges(A,B,C,D,E,F,G,H,I,J,A)
  \end{tikzpicture}
}

\newcommand{\drawCfiveCFI}{%
  \begin{tikzpicture}[scale=0.8,rotate=0, baseline={(0,0)}]%
    \GraphInit[vstyle=Classic]%
    \graphpreamble
    \SetGraphUnit{1}
    \begin{scope}[scale=0.6,shift={(0,1.5)}]
        \Vertices{circle}{A,B,C,D,D2}
        \Edges(A,B,C,D,D2,A)
    \end{scope}
    \begin{scope}[scale=0.6, shift={(1,-1)}]
      \Vertices{circle}{E,F,G,H,H2}
      \Edges(E,F,G,H,H2,E)
    \end{scope}
  \end{tikzpicture}
}

\newcommand{\drawAnchoredCfour}{%
  \adjustbox{}{
  \begin{tikzpicture}[scale=0.7,rotate=0, baseline={(0,-0.15)}]%
    \GraphInit[vstyle=Classic]%
    \graphpreamble
      \SetVertexLabel
    \SetGraphUnit{1}
    \Vertex[x=0, y=0, NoLabel]{A}
    \Vertex[x=0, y=1,L=2, Lpos=180]{B}
    \Vertex[x=1, y=1,L=3]{C}
    \Vertex[x=1, y=0,L=4]{D}
    \Edges(A,B,C,D,A)%
    \node[fill=white, circle, minimum size=0pt, inner sep=-1pt, outer sep=0pt] at (0,0) {$\anchor$};
  \end{tikzpicture}
  }
}

\newcommand{\drawAnchoredCfourB}{%
  \adjustbox{}{
  \begin{tikzpicture}[scale=0.5,rotate=0, baseline={(0,-0.0)}]%
    \GraphInit[vstyle=Classic]%
    \graphpreamble
    \SetGraphUnit{1}
    \Vertex[x=0, y=2]{A}
    \Vertex[x=0, y=1]{B}
    \Vertex[x=0, y=0]{C}
    \Edges(A,B,C)%
    \node[fill=white, circle, minimum size=0pt, inner sep=-1pt, outer sep=0pt] at (0,0) {$\anchor$};
  \end{tikzpicture}
  }
}

\newcommand{\drawAnchoredCfourC}{%
  \adjustbox{}{
  \begin{tikzpicture}[scale=0.5,rotate=0, baseline={(0,-0.0)}]%
    \GraphInit[vstyle=Classic]%
    \graphpreamble
    \SetGraphUnit{1}
    \Vertex[x=0, y=2]{A}
    \Vertex[x=0, y=1]{B}
    \Vertex[x=0, y=0]{C}
    \Edges(A,B,C)%
    \node[fill=white, circle, minimum size=0pt, inner sep=-1pt, outer sep=0pt] at (0,1) {$\anchor$};
  \end{tikzpicture}
  }
}

\newcommand{\drawAnchoredCfourD}{%
  \adjustbox{}{
  \begin{tikzpicture}[scale=0.5,rotate=0, baseline={(0,-0.0)}]%
    \GraphInit[vstyle=Classic]%
    \graphpreamble
    \SetGraphUnit{1}
    \Vertex[x=0, y=1]{B}
    \Vertex[x=0, y=0]{C}
    \Edges(B,C)%
    \node[fill=white, circle, minimum size=0pt, inner sep=-1pt, outer sep=0pt] at (0,0) {$\anchor$};
  \end{tikzpicture}
  }
}

\newcommand{\drawKthreePlusTwo}{%
\adjustbox{}{
  \begin{tikzpicture}[scale=0.21,rotate=0,baseline={(0,0.25)}]%
    \GraphInit[vstyle=Classic]%
    \graphpreamble
    \SetGraphUnit{1}
    \Vertex[x=0, y=2]{A}
    \Vertex[x=1.35, y=2]{B}
    \Vertex[x=0.75, y=1]{C}
    \Vertex[x=2, y=2.5]{D}
    \Vertex[x=2, y=1.2]{E}
    \Edges(A,B,C,A)%
    \Edges(B,E)
    \Edges(B,D)
  \end{tikzpicture}
  }
}

\newcommand{\drawKthreeplusPtwo}{%
  \begin{tikzpicture}[scale=0.16,rotate=-60,baseline={(0,-0.10)}]%
    \GraphInit[vstyle=Classic]%
    \graphpreamble
    \SetGraphUnit{1}
    \Vertex[x=0, y=2]{A}
    \Vertex[x=1.5, y=2]{B}
    \Vertex[x=0.75, y=1]{C}
    \Vertex[x=0.75, y=0]{D}
    \Vertex[x=0.75, y=-1]{E}
    \Edges(A,B,C,A)%
    \Edges(C,D,E)
  \end{tikzpicture}
}

\newcommand{\drawStarthreePlusOne}{%
  \begin{tikzpicture}[scale=0.12, baseline={(0,-0.06)}]%
    \GraphInit[vstyle=Classic]%
    \graphpreamble
    \SetGraphUnit{1}
    \Vertex[x=0, y=0]{A}
    \Vertex[x=1.2, y=-1]{B}
    \Vertex[x=0.0, y=1.3]{C}
    \Vertex[x=-1.2, y=-1]{D}
    \Vertex[x=2.4, y=-1]{E}
    \Edges(A,B,E)%
    \Edges(A,C)%
    \Edges(A,D)%
  \end{tikzpicture}
}

\newcommand{\drawCfourPlusOne}{%
  \begin{tikzpicture}[scale=0.21,rotate=0, baseline={(0,0.05)}]%
    \GraphInit[vstyle=Classic]%
    \graphpreamble
    \SetGraphUnit{1}
    \Vertex[x=0, y=0]{A}
    \Vertex[x=0, y=1]{B}
    \Vertex[x=1, y=1]{C}
    \Vertex[x=1, y=0]{D}
    \Vertex[x=1.75, y=1.75]{E}
    \Edges(A,B,C,D,A)%
    \Edges(C,E)
  \end{tikzpicture}
}

\newcommand{\drawAlmostKfour}{%
  \begin{tikzpicture}[scale=0.24,rotate=0, baseline={(0,0.05)}]%
    \GraphInit[vstyle=Classic]%
    \graphpreamble
    \SetGraphUnit{1}
    \Vertex[x=0, y=0]{A}
    \Vertex[x=0, y=1]{B}
    \Vertex[x=1, y=1]{C}
    \Vertex[x=1, y=0]{D}
    \Edges(A,B,C,D,A)%
    \Edges(A,C)
  \end{tikzpicture}
}\clearpage{}%
\usepackage{hyperref}

\usepackage[accepted]{icml2024}

\usepackage{amsmath}
\usepackage{amssymb}
\usepackage{mathtools}
\usepackage{amsthm}
\usepackage{comment}
\usepackage[inline]{enumitem}
\usepackage[capitalize,noabbrev]{cleveref}

\theoremstyle{plain}
\newtheorem{theorem}{Theorem}[section]
\newtheorem{proposition}[theorem]{Proposition}
\newtheorem{lemma}[theorem]{Lemma}

\theoremstyle{definition}
\newtheorem{definition}[theorem]{Definition}

\theoremstyle{remark}

\newtheorem{example}[theorem]{Example}

\usepackage[textsize=tiny]{todonotes}

\icmltitlerunning{Homomorphism Counts for Graph Neural Networks: All About That Basis}

\usepackage{amsmath,amsfonts,bm}

\def\eqref#1{equation~\ref{#1}}

\def\1{\bm{1}}

\def\vh{{\bm{h}}}

\def\vz{{\bm{z}}}

\DeclareMathAlphabet{\mathsfit}{\encodingdefault}{\sfdefault}{m}{sl}
\SetMathAlphabet{\mathsfit}{bold}{\encodingdefault}{\sfdefault}{bx}{n}

\def\sR{{\mathbb{R}}}

\begin{document}

\twocolumn[
\icmltitle{Homomorphism Counts for Graph Neural Networks: All About That Basis}

\icmlsetsymbol{equal}{*}

\begin{icmlauthorlist}
\icmlauthor{Emily Jin}{oxford}
\icmlauthor{Michael Bronstein}{oxford}
\icmlauthor{$\dot{\text{I}}$smail $\dot{\text{I}}$lkan Ceylan}{oxford}
\icmlauthor{Matthias Lanzinger}{vienna,oxford}
\end{icmlauthorlist}

\icmlaffiliation{oxford}{Department of Computer Science, University of Oxford, Oxford, UK}
\icmlaffiliation{vienna}{Institute for Logic and Computation, TU Wien, Vienna, AT}

\icmlcorrespondingauthor{Emily Jin}{emily.jin@cs.ox.ac.uk}

\icmlkeywords{graph neural networks, homomorphism counting, expressivity}

\vskip 0.3in
]

\printAffiliationsAndNotice{}  %

\begin{abstract}
A large body of work has investigated the properties of graph neural networks and identified several limitations, particularly pertaining to their expressive power. Their inability to count certain \emph{patterns} (e.g., cycles) in a graph lies at the heart of such limitations, since many functions to be learned rely on the ability of counting such patterns. Two prominent paradigms aim to address this limitation by enriching the graph features with \emph{subgraph} or \emph{homomorphism} pattern counts. In this work, we show that both of these approaches are sub-optimal in a certain sense and argue for a more \emph{fine-grained} approach, which incorporates the homomorphism counts of \emph{all} structures in the ``basis'' of the target pattern. This yields strictly more expressive architectures without incurring any additional overhead in terms of computational complexity compared to existing approaches. We prove a series of theoretical results on node-level and graph-level \emph{motif parameters} and empirically validate them on standard benchmark datasets. 
\end{abstract}

\section{Introduction}

Graph neural networks (GNNs) \citep{scarselli2008graph,Gori2005} are a class of architectures for learning invariant functions on graphs. Their success in various domains~\citep{Shlomi21,DuvenaudMABHAA15,ZitnikAL18} has led to a number of architectures~\citep{Kipf16,xu18,VelickovicCCRLB18,HamiltonYL17}.  Most of the architectures used in practice fall under the framework of message passing neural networks (MPNNs) from \citet{GilmerSRVD17}, where the core idea is to iteratively update each node's representation based on the messages received from the node's neighbors.

\begin{figure}[!ht]
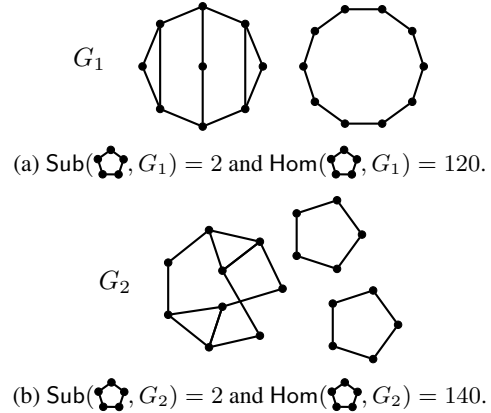

    \centering
    \begin{subfigure}{0.4\textwidth}
    \centering
        $G_1$  
        \quad
        \drawLthreeonetwist
        \quad
        \drawCfiveTwist
    \caption{$\subs(\drawCfive, G_1)=2$ and $\homs(\drawCfive, G_1)=120$.}
    \label{subfig:g1}
    \end{subfigure}
    
    \vspace{0.2cm}
    \begin{subfigure}{0.4\textwidth}
    \centering
        $G_2$
        \quad
       \drawLthreeoneCFI 
       \drawCfiveCFI
    \caption{$\subs(\drawCfive, G_2)=2$ and $\homs(\drawCfive, G_2)=140$.}
    \label{subfig:g2}
    \end{subfigure}
    \caption{Two 1-WL indistinguishable graphs $G_1$ and $G_2$. These graphs have the same number of $5$-cycles, but they can be distinguished by the homomorphism counts of $5$-cycles. $\subs(F,G)$ is the number of times $F$ occurs as a subgraph in $G$, and $\homs(F,G)$ is the number of homomorphisms from $F$ to $G$.}
    \label{fig:subgraphs.not.enough}
    \vspace{-1.5em}
\end{figure}

\textbf{Graph Motif Parameters for Counting Patterns.} The expressive power of MPNNs is upper bounded by the \emph{1-dimensional Weisfeiler Leman graph isomorphism test (1-WL)}~\cite{xu18,MorrisAAAI19}. Hence, all known inexpressiveness results for 1-WL apply to MPNNs. In particular, this implies that MPNNs cannot express a wide variety of \emph{graph motif parameters}, which include functions that count the occurrences of  basic graph patterns, such as paths or cycles. Since these patterns are abundant in real-world domains (e.g., cycles in molecules), a large body of work proposes enhancing the input graphs with such pattern counts. 
One line of work enriches the graph features with \emph{subgraph} counts \cite{BouritsasFZB23,BevilacquaFLSCBBM22,Fra+2022}, while another enriches the graph features with \emph{homomorphism} counts \cite{BarceloGRR21}. Although both approaches increase the expressive power of standard MPNNs, it remains unclear how these approaches compare to one another in terms of expressiveness gained. 

\begin{example}
\label{ex:sub.not.better}
    Consider the graphs $G_1$ and $G_2$ from \Cref{fig:subgraphs.not.enough}, which are indistinguishable by $1$-WL and both contain \emph{two} 5-cycles as \emph{subgraphs}: formally we write $\subs(\drawCfive,G_1)=\subs(\drawCfive,G_2)=2$. Hence, adding subgraph counts for $5$-cycles to 1-WL does not help in distinguishing these graphs, but adding homomorphism counts for $5$-cycles does help in  distinguishing these graphs since $\homs(\drawCfive,G_1)\neq \homs(\drawCfive,G_2)$. $\exend$
\end{example}

This example presents two 1-WL-indistinguishable graphs, where injecting (i.e. enriching with) homomorphism counts is a better strategy than injecting subgraph counts. How about the other direction? Are there graph pairs where injecting subgraph counts is the superior approach? 
\begin{example}
\label{ex:homnotbetter}
Let $H_1$ be the first component of $G_1$ (left of \Cref{subfig:g1}) and $H_2$ be the first component of $G_2$ (left of \Cref{subfig:g2}). It is easy to verify that $H_1$ and $H_2$ cannot be distinguished by 1-WL, and we further observe that:
\[
\homs(\drawCfive,H_1)= \homs(\drawCfive,H_2)=120.
\]
Thus, adding homomorphism counts for \drawCfive does not help in distinguishing these graphs. Conversely, we have $\subs(\drawCfive,H_1)=2$ and $\subs(\drawCfive,H_2)=0$, which suffices to distinguish the graphs. $\exend$
\end{example}

These examples show that these two approaches are incomparable in terms of the expressiveness gain that they yield. This brings us to the main quest of this paper: what information can we inject in GNNs that strictly subsumes both approaches? In this paper, we build on recent advances in graph theory that show that \emph{graph motif parameters}, which encapsulate many functions of interest, can be expressed as a \emph{linear combination of homomorphism counts}.

\begin{example}
\label{ex:linear-comb}
    Consider our running example from \Cref{fig:subgraphs.not.enough}. We can express $\subs(\drawCfive, G)$ as the linear sum: 
    \begin{multline*}
     \frac{1}{10} \homs(\drawCfive, G) -  \frac{1}{2}\homs(\drawKthreeplusone, G) + \frac{1}{2} \homs(\drawKthree, G) 
\end{multline*}
where the terms for \drawCfive, \drawKthreeplusone, and \drawKthree define the \emph{homomorphism basis} of the target function $\subs(\drawCfive, \cdot)$.     $\exend$
\end{example}

In a nutshell, the counts of the terms from the homomorphism basis are the key factors for expressiveness: we can access the homomorphism counts of all ``base terms'' that are \emph{necessary and sufficient} to express the counts of the target pattern. In this sense, the target expressiveness gain is a very precise and a fine-grained one. Let us illustrate on our example how the basis terms are obtained.

\begin{example}
\label{ex:basis}
    Consider the example from \Cref{fig:subgraphs.not.enough}. The graphs in the basis of the target pattern \drawCfive are those that are obtained by contracting non-adjacent vertices:
    \vspace{-0.5em}
    \begin{figure}[h!]
    \centering
    \drawCfiveBigZ \\

    \vspace{0.3em}

    \drawCfiveBigA \quad {\huge $\rightarrow$} \quad \drawKthreeplusoneBigA \\

    \vspace{0.3em}

    \drawCfiveBigB  {\huge \;$=$\!\! } \drawKthreeplusoneBigB  {\huge $\rightarrow$}   \drawKthreeBigB
    \end{figure}
    
    which results in a richer set of patterns in the basis.
    $\exend$
\end{example}

Computing the homomorphism count for all basis terms of a target pattern is strictly more expressive than subgraph and homomorphism counting, as shown in \Cref{fig:overview}. A very appealing aspect of our approach is that it does not incur any additional overhead in terms of computational complexity compared to subgraph counting. In fact, the most efficient way of computing subgraph counts is by computing the linear combination of homomorphism counts.
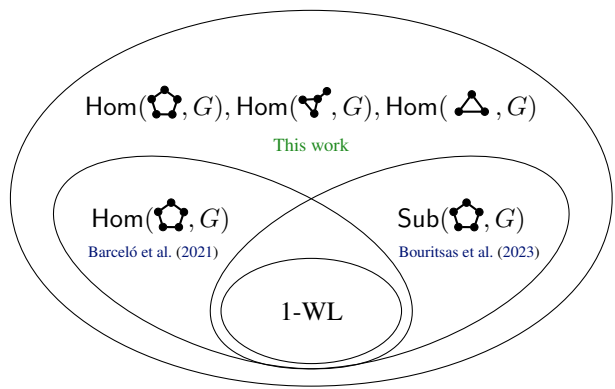
\begin{figure}[t]
\centering
    \begin{tikzpicture}[scale=1.00]
        \draw (0,0) ellipse (12mm and 7mm);
        \node (0,0.5) {1-WL};

    \begin{scope}[rotate=20]
            \draw (1.2,0.25) ellipse (25mm and 12mm);
    \end{scope}
    \node[anchor=south] (0) at (2.0, 0.9) {$\subs(\drawCfive, G)$};
    \node[anchor=south] (0) at (2.1, 0.55) {\tiny\citet{BouritsasFZB23}};

    \begin{scope}[rotate=-20]
            \draw (-1.2,0.25) ellipse (25mm and 12mm);
    \end{scope}
    \node[anchor=south] (0) at (-2.0, 0.9) {$\homs(\drawCfive, G)$};
\node[anchor=south] (0) at (-2.1, 0.55) {\tiny\citet{BarceloGRR21}};

    \draw (0,1.5) ellipse (40mm and 25mm);
    \node[anchor=south] (0) at (-0, 2.4) {$\homs(\drawCfive, G), \homs(\drawKthreeplusone, G), \homs(\drawKthree, G) $};
    \node[anchor=south] (0) at (0, 2) {\scriptsize \color{ForestGreen}This work};
    \end{tikzpicture}
\caption{Expressiveness gain from injecting parameters. All inclusions are proper. All three features require (asymptotically) equivalent effort to calculate; they all can be computed in quadratic time (and no better).}
\label{fig:overview}
\end{figure}

\textbf{Graph Motif Parameters for Beyond Counting Patterns.} Although our exposition primarily focuses on subgraph and homomorphism counts for GNNs, our main objective is to consider the overarching class of graph motif parameters, which has broad implications beyond fixed pattern counting.

\begin{example}
\label{ex:logic}
    We can %
    express whether an atom (i.e., node) in a molecular dataset is part of an alcohol group with the following first-order formula:
    \begin{align*}
       \psi(x,y,z) &:= (\mathsf{bond}(x,y) \land \mathsf{bond}(y,z) \land \mathsf{O}(y) \land \mathsf{H}(z)) 
    \end{align*}
    Similarly, we can 
    express whether an atom is part of an acyl halide functional group as:
    \begin{align*}
        \tau(x,y,z,h) &:= (\mathsf{bond}(x,y) \land \mathsf{dbond}(y,z) \land \mathsf{bond}(y,h) \land \\
        & \mathsf{C}(y) \land \mathsf{O}(z) \land (\mathsf{F}(h) \lor \mathsf{Cl}(h) \lor \mathsf{Br}(h) \lor \mathsf{I}(h)))
    \end{align*}
    where  $\mathsf{O}, \mathsf{H}, \mathsf{C}, \mathsf{F}, \mathsf{Cl}, \mathsf{Br}, \text{and } \mathsf{I}$ represent corresponding elements in the periodic table. We can also disjunctively combine these properties to express that an atom is part of an alcohol or acyl halide functional group:
     \begin{align*}
        \phi(x,y,z,h):= \psi(x,y,z) \lor \tau(x,y,z,h) 
    \end{align*}
    The number of satisfying assignments of $\phi(x,y,z,h)$ is a graph motif parameter that can be expressed as a linear combination of homomorphism counts. This can be done at node-level and for specific variables, e.g., how often variable $x$ is mapped to node $v$ in a satisfying assignment. $\exend$
\end{example}

\Cref{ex:logic} captures a general property that goes beyond fixed pattern counting, yet our method can still precisely identify which homomorphism counts need to be injected in order to lift a model to the desired level of expressivity. By considering the broader framework of graph motif parameters, we provide an exact and principled way of achieving targeted expressivity gains for GNN architectures.

\textbf{Contributions.} 
Our main contributions can be summarized as follows: 
\begin{itemize}[leftmargin=15pt,topsep=4pt]
    \item We study the expressivity of GNNs with respect to injecting graph motif parameters. We show that, in terms of expressiveness, it is highly beneficial to inject the homomorphism basis of a parameter instead of the parameter itself (\Cref{thm.moreexpr}), even for higher-order GNNs. 
    \item This gain in expressiveness comes with additional benefits. In particular, we show how our method can naturally avoid redundant information (\Cref{sec:missing-piece}), and that it applies both on graph- and node-level (\Cref{sec:node.level}). 
   \item 
    Our approach provides a flexible framework for injecting information far beyond the fixed pattern counts considered in previous work, e.g., graphlet counts or logical properties (\Cref{sec:beyond}). 
    \item We empirically validate the theoretical findings and demonstrate efficacy of the approach on a wide variety of benchmarks (\Cref{sec:experiments}). We observe state of the art results in expressivity and significant improvements in real-world regression and link-prediction tasks.
\end{itemize}
From a broader perspective, this work strengthens the connections between recent graph-theoretical results and graph machine learning, informing future work for the targeted design of GNN architectures suitable for a domain of interest.
\section{Related Work}
It is well-known that the expressive power of MPNNs is upper bounded by 1-WL in terms of graph distinguishability~\cite{xu18,MorrisAAAI19}. Models such as graph isomorphism network (GIN)~\cite{xu18} can match this bound, whereas other popular architectures such as graph convolutional networks (GCNs)~\citep{Kipf16}, or graph attention networks (GATs)~\citep{VelickovicCCRLB18} cannot.

This expressiveness limitation has motivated a large body of work in the graph machine learning literature, including the study of  \emph{higher-order} GNN architectures~\cite{MorrisAAAI19,MaronBSL19, MaronFSL19,KerivenP19} which generally match the expressive power of the $k$-WL algorithm for some $k>1$.   Some works consider graphs with \emph{unique node features}~\cite{Loukas20}, and others augment the node features with \emph{random features}~\cite{SatoSDM2020, AbboudCGL21} to achieve higher expressive power. Expressive power of GNNs has also been evaluated in terms of their ability to detect certain graph components, such as biconnected components \cite{Zha+2023}. Building on the classical literature of isomorphism testing and exploiting graph components (and associated graph decompositions) has led to efficient, isomorphism-complete learning algorithms for planar graphs \cite{Dimitrov2023}.

There are different means for studying expressive power, and our work's primary focus is on measuring expressive power based on the ability to capture graph motif parameters. MPNNs are very limited in terms of their expressive power when it comes to counting patterns and so any function of this form goes beyond the ability of these architectures \cite{chen2020can}. This led to a rich line of work of GNN architectures that inject such pattern counts explicitly \cite{BouritsasFZB23,BarceloGRR21,BevilacquaFLSCBBM22,Fra+2022}, some of which are known as subgraph GNNs. \citet{zhang2023complete} also recently provided an expressiveness hierarchy for subgraph GNNs via the so-called subgraph Weisfeiler-Leman tests. 

\citet{papp2022} investigated the expressiveness of specific properties, such as counts for all $k$-vertex substructures. They presented some specific results on the expressiveness gain such as, showing that GNNs with all counts for $k$-vertex substructures can count $k$-cycles, but not $k-1$ cycles. Our framework provides unified answers for these specific expressivity questions.  \citet{BarceloGRR21} presented one of the early works highlighting the importance of using homomorphism counts with GNNs. Very recently, \citet{welke} propose a randomised approach that uses homomorphism counts to distinguish graphs \emph{in expectation}. The work of \citet{zhang2024beyond} is closely related to ours, where it has been shown that the expressiveness of  GNN architectures under consideration can be characterized in terms of homomorphism counts. This further motivates our approach as will be discussed in \Cref{sec:missing-piece}.

Other works have studied homomorphism counts for more traditional machine learning models, such as support vector machines (SVMs) and random forest classifiers \cite{nguyen2020graph, wolf2023structural}. \citet{Grohe2017} also investigated using homomorphism vectors to represent graphs.

\section{Background}
\textbf{Graphs, Homomorphisms, and Invariants.}
A (undirected) graph $G$ is a tuple $(V,E)$ where $V$ is a set of \emph{vertices} and $E \subseteq V \times V$ is a symmetric \emph{edge} relation. In some contexts we write $V(G)$ and $E(G)$ for the vertices and edges, respectively, of graph $G$. We say that $G'$ is a \emph{subgraph} of $G$ if $V(G') \subseteq V(G)$ and $E(G') \subseteq E(G)$. The \emph{induced subgraph} of $G$ by $U \subseteq V(G)$ is the graph $G[U]$ with vertices $U$ and exactly those edges from $G$ where both ends of the edge are in $U$.
Let $\rho$ be a partition of $V(G)$. The quotient graph $G/\rho$ is the graph where the vertices in each block are contracted into a single vertex, i.e., the vertices of $G/\rho$ are the blocks $B \in \rho$, and there is an edge $(B_1, B_2)$ iff there is an edge $(v,u)$ in $G$ for any $v \in B_1, u \in B_2$.

A \emph{homomorphism} from graph $G$ to graph $H$ is a function $h : V(G) \to V(H)$ such that for every $(v,w) \in E(G)$, $(h(v), h(w)) \in E(H)$. An \emph{isomorphism} between graphs $G$ and $H$, denoted $G \isomorphic H$, is a bijection $f : V(G) \to V(H)$ such that $(v,w) \in E(G)$ if and only if $(f(v),f(w)) \in E(H)$, for each $v,w \in V(G)$. 
 We write  $\homs(G,H)$ for the number of homomorphisms from $G$ to $H$, $\subs(G,H)$ for the number of subgraphs $H'$ of $H$ such that $H'\isomorphic G$, and $\indsubs(G, H)$ for the number of sets $U \subseteq V(H)$ such that $H[U] \isomorphic G$, i.e., the number of times $G$ is an (induced) subgraph of $H$. By $\homs(G, \cdot)$, we denote the function that maps graph $H$ to $\homs(G,H)$, and analogously for $\subs$.  For a set of graphs $\mathbb{G}$, we write $\homs(\mathbb{G})$ to represent the set of functions $\{\homs(G, \cdot) \mid G \in \mathbb{G}\}$.

A \emph{graph invariant} $\xi$ is a function on graphs such that for all graphs $G, H$, and all isomorphisms $f$ between $G$ and $H$, we have $\xi(G) = \xi(H)$. A graph invariant \emph{distinguishes} two graphs $G$ and $H$, denoted $G \not\equiv_\xi H$, if  $\xi(G) \neq \xi(H)$, and conversely, we write $G\equiv_\xi H$ if $\xi(G) = \xi(H)$.

\textbf{Graph Neural Networks.}  
Let us consider featured graphs $G=(V, E, \zeta)$ where $\zeta: V(G)\to \sR$, and so $\zeta(u)$ denotes the feature of a node $u \in V$. All notions introduced earlier extend to featured graphs in the usual way. 
We focus on \emph{Message-Passing Neural Networks (MPNNs)}~\citep{GilmerSRVD17} that encapsulate the vast majority of GNNs. 
An MPNN updates the initial node representations $\vh_{v}^{(0)}=\zeta(v)$ of each node $v$ for $0 \leq \ell < L$ iterations based on its own state and the state of its neighbors $\mathcal{N}_v$ as:
\begin{align*}
	\vh_v^{\ell+1} =\update^{(\ell)}  \left(
	\vh_v^\ell, 
	\aggregate^{(\ell)}\left(\vh_v^\ell, \{\!\!\{\vh_u^\ell\mid u\in\mathcal{N}_v\}\!\!\}\right)\right),
\end{align*}
where $\{\!\!\{\cdot\}\!\!\}$ denotes a multiset and $\update^{(\ell)}$ and $\aggregate^{(\ell)}$ are differentiable \emph{update} and \emph{aggregation} functions, respectively. We denote by $d^{(\ell)}$ the dimension of the node embeddings at iteration (layer) $\ell$.  
The final representations $\vh_v^{\left(L\right)}$ of each node $v$ can be used for predicting node-level properties, or they can be pooled to form a graph embedding vector $\vz_G$, which can be used for predicting graph-level properties.  

\textbf{Weisfeiler-Leman Hierarchy.} We will refer to the 1-WL algorithm~\cite{weisfeiler1968reduction}, as well as to the usual $k$-WL hierarchy, which are iterative algorithms that yield graph invariants. We will always refer to the \emph{folklore} version of these algorithms ($k$-WL), rather than the oblivious version (oblivious $k$-WL), where the only difference is in the dimension counts: $k$-WL is equivalent to oblivious $(k+1)$-WL, for every $k>1$ \cite{Grohe2017,GroheOtto15}.

\textbf{Graph Motif Parameters.}
A \emph{graph parameter} is a function that maps graphs into $\mathbb{Q}$.
A \emph{graph motif parameter}~\citep{CurticapeanDM17} is a graph parameter that can be expressed as a finite linear combination of homomorphism counts into $G$. Formally, $\Gamma$ 
is a  graph motif parameter, if 
there are graphs $F_1,\dots,F_\ell$ and a function $\alpha$ mapping them to $\mathbb{Q} \setminus \{0\}$, such that 
for every  graph $G$: 
\begin{equation} \label{lgmp}
    \Gamma(G) \, = \, \sum_{i=1}^\ell \alpha(F_i)\cdot  \homs(F_i, G). 
\end{equation}
We say that a motif parameter $\Gamma$ is \emph{connected} if every graph in $\supp(\Gamma)$ is connected. Informally we also refer to the graphs in the support as the \emph{homomorphism basis} of $\Gamma$. 

For any graph $F$, $\subs(F, \cdot)$ is a graph motif parameter. Its support, commonly called $\spasm(F)$, is the set of all quotient graphs of $F$ (up to isomorphism)~\citep{Lovasz12}. When the host graphs are assumed to be loop-free, the $\spasm$ is implicitly understood to only contain the loop-free quotients.   If $F$ is connected, $\subs(F,\cdot)$ is a connected graph motif parameter.  As an example, recall the quotients of $\drawCfive$ as illustrated in \Cref{ex:basis}%
\footnote{We omit the technical discussion of how precisely the coefficients are obtained; refer instead to \citet{CurticapeanDM17}.}.

\textbf{Expressing a Graph Parameter.}
 We say that a GNN architecture $\gnn$ can \emph{express} a graph parameter $f$ if for every pair of graphs $G,H$, there exists a model parametrization $\gnnmodel$ such that $G \equiv_{\gnnmodel} H$ implies $f(G) = f(H)$.
In other words, there can be no situation where two graphs are equivalent to the model $\gnnmodel$ but $f$ is different on the two graphs. 
We are interested in the question of how the expressiveness gained by different sets of additional features compare to one another. Thus, we extend our notion of expressiveness to the case where $\gnn$ additionally has access to a set of additional feature maps $\extraparams$, i.e., invariant functions that map the input graph to a rational number. 
We say that a GNN architecture $\gnn$ with $\extraparams$ can \emph{express} a graph parameter $f$ if for every pair of graphs $G$, $H$,  there exists a model parametrization $\gnnmodel$ such that if $G \equiv_{\gnnmodel} H$ and $a(G)=a(H)$ for all $a \in \extraparams$, then $f(G)=f(H)$. 
For sets of features $\extraparams, \extraparams'$ and GNN $\gnn$, we say that $\gnn$ with $\extraparams$ is \emph{at least as expressive} as $\gnn$ with $\extraparams'$ if the former expresses all functions that are expressed by the latter. Similarly, $\gnn$ with $\extraparams'$ is \emph{strictly more expressive} than $\gnn$ with $\extraparams$ if the former is at least as expressive and can express strictly more functions than the latter.

\section{All About the Homomorphism Basis}
\label{sec:all}
We study the expressivity of GNNs with respect to graph motif parameters and propose an efficient method for injecting such parameters with provably stronger expressiveness in addition to a variety of further benefits.

\subsection{It Is More Expressive}
\label{sec:more.exp}

Recall that $\extraparams = \{\subs(\drawCfive, \cdot)\}$ and $\extraparams' = \{\homs(\drawCfive, \cdot)\}$ are incomparable if added to 1-WL expressiveness (\Cref{ex:sub.not.better}). 
This also holds in general for arbitrary patterns that have more than one acyclic graph in their basis (follows from \Cref{thm.moreexpr}). 
At the same time, 1-WL with either kind of additional feature can be more expressive than plain 1-WL, 
as previously observed by \citet{BouritsasFZB23} and \citet{BarceloKM0RS20}.
We show that the homomorphism basis of $\subs(\drawCfive, \cdot)$ is more expressive than either approach. %
In fact, we prove a much more general and stronger statement. For (almost) every connected graph motif parameter $\Gamma$, providing the homomorphism basis $\extrabasis_\Gamma$ as additional features is \emph{strictly} more expressive than injecting $\Gamma$.

\begin{theorem}
\label{thm.moreexpr}
Let $\Gamma$ be a connected graph motif parameter and $k \geq 1$. 
    Then $k$-WL with $\extrabasis_\Gamma$ is at least as expressive as $k$-WL with $\{\Gamma\}$.
    Moreover, if at least two functions in $\homs(\mathsf{Supp}(\Gamma))$ cannot be expressed by $k$-WL with $\{\Gamma\}$, then $\gnn$ with $\extrabasis_\Gamma$ is strictly more expressive.\footnote{We prove a significantly stronger statement showing that this applies to many common GNNs rather than just $k$-WL. See \Cref{app:moreexpr} for details.}
\end{theorem}

We wish to emphasize that the gain in expressiveness observed in \Cref{thm.moreexpr} is not only a matter of narrow theoretical increase. This is illustrated well by subgraph counting. Recall the basis of $\Gamma = \subs(\drawCfive, \cdot)$ from \Cref{ex:basis}. There, $1$-WL with $\extrabasis_\Gamma$ does not only distinguish $\subs(\drawCfive, \cdot)$ but also $\subs(\drawKthreeplusone, \cdot)$ and $\subs(\drawKthree, \cdot)$.
 In general, we observe the following.
\begin{proposition}
\label{freesubgraphs}
    Let $\gnn$ be a GNN, let $G$ be a graph, and let $F \in \spasm(G)$. Then 
    $\subs(F, \cdot)$ can be expressed by $\gnn$ with $\extrabasis_{\subs(G, \cdot)}$.
\end{proposition}

\subsection{It Is Computationally Efficient}
\label{sec:efficient}
In practice, the additional features given by $\extraparams$ to a GNN must be computed externally for each input graph. \citet{CurticapeanDM17} have shown that, in slightly simplified terms, the most efficient way to compute $\Gamma(G)$ is to compute the terms of the homomorphism basis and then obtain $\Gamma(G)$ via \Cref{lgmp}.
Hence, providing a GNN with $\extrabasis_\Gamma$ is not just more expressive but also computationally no more effort than providing $\Gamma$ itself.
The precise technical statements are intricate as they require tools from parameterized complexity theory; we refer the interested reader to \Cref{app:complexity} for formal details.

Popular subgraph counting implementations often exhibit significantly worse complexity than required to compute $\homs(\supp(\Gamma))$. Typically, to compute $\subs(F, G)$ they have running times exponential in $|V(G)|$ (cf., \citet{RibeiroPSAS21}). Earlier works on injecting structural information in GNNs focus on particular settings where this worst-case behavior is not observed: sparse input graphs such as molecules and special patterns like cycles allow for good ad hoc algorithms for such graphs. For slightly more complex graphs or patterns, direct computation of $\subs$ becomes intractable and computation via the homomorphism basis  becomes much more efficient also in practice.

\subsection{It Avoids Redundancy}
\label{sec:missing-piece}
While expressiveness is a natural criterion to consider with respect to injecting additional features, we also want to avoid redundant information. In this regard, the homomorphism basis is ideally suited as the expressiveness of GNNs with respect to homomorphism counts is well-understood. In practice, this means that instead of providing all of $\extrabasis_\Gamma$ as additional features to distinguish some target function $\Gamma$, we can often precisely determine which functions in $\extrabasis_\Gamma$ are not already distinguished by the GNN and only inject those.

Recent results provide a clear picture of when a GNN can distinguish a function $\homs(F, \cdot)$. \citet{neuen2023homomorphism} recently showed that ($k-1$)-WL cannot distinguish any function $\homs(F,\cdot)$ when $k$ is the \emph{treewidth}
of $F$. A classic result by \citet{Dvorak10} states that $k$-WL is enough to distinguish all such functions. \citet{geerts2022expressiveness} showed an upper bound by $k$-WL for pattern counting based on the treewidth of patterns in the \spasm. This has significant practical implications: for architectures such as GIN~\citep{xu18} that reach 1-WL expressiveness, injecting only the cyclic graphs of $\extrabasis_\Gamma$ gives equivalent expressiveness to injecting all of $\extrabasis_\Gamma$.
On the other hand, architectures like GAT~\cite{VelickovicCCRLB18} and GCN~\cite{Kipf16} are weaker than 1-WL and thus motivate providing all of $\extrabasis_\Gamma$ as additional features to the network.
Recently, \citet{zhang2024beyond} gave precise definitions for the sets of functions $\homs(F,\cdot)$ that can (and cannot) be expressed by Local $k$-GNNs, and Subgraph GNNs, thus showing precisely what part of $\extrabasis_\Gamma$ can already be expressed by the corresponding GNN architectures.

Furthermore, it is never necessary to consider homomorphism counts from disconnected graphs. Note that this is not trivial but particular to homomorphism counting, e.g., no similar simplification can be made for subgraph counting. See \Cref{sec:beyond} below for an application of this observation.

\begin{proposition}
\label{onlycon}
    Let $\gnn$ be a GNN, let $\mathbb{G}$ be a set of graphs and let $\mathsf{CC}(\mathbb{G})$ be the set of the (maximal) connected components in $\mathbb{G}$. Then $\gnn$ with $\homs(\mathsf{CC}(\mathbb{G}))$ is at least as expressive as $\gnn$ with $\homs(\mathbb{G})$.
\end{proposition}

\subsection{It Works for Node-level Subgraph Information}
\label{sec:node.level}
Previous work has shown the efficacy of providing subgraph counts at node-level. We will use the term \emph{anchored graph} to mean a graph $G$ with a marked vertex $\anchor$. When $G$ is anchored, we write $\subs(G,H,v)$ for the number of subgraphs $H'$ of $H$ where $G \isomorphic H'$ such that the marked vertex maps to $v$ in the isomorphism.
In particular, 
\citet{BouritsasFZB23} propose labeling each node $v \in H$ with $\subs(G',H,v)$ for every way $G'$ of marking a graph $G$.

Local homomorphism counts of the graphs in the \spasm are not enough to determine local subgraph counts (see \Cref{app:nodelevel}). However, in the following we show how the framework we have used so far naturally extends to this setting.
For an anchored graph $G$, the anchor of quotient $G/\rho$ is the vertex for the block $B \in \rho$  that contains the anchor of $G$. Analogously to the standard case, $\spasm^\anchor(G)$ is the set of all (loop-free) quotient graphs of $G$. 

\begin{example}
The $\spasm(\drawCfour)$ contains $\drawCfour$, $\drawPtwo$, and $\drawPone$, whereas $\spasm^\anchor(\drawCfour)$ consists of the four graphs:
    \begin{figure}[H]
    \centering
    \drawAnchoredCfour
    \quad
    \drawAnchoredCfourB
    \quad
    \drawAnchoredCfourC
    \quad
    \drawAnchoredCfourD
\end{figure}\vspace{-0.4cm}
In $\spasm^\anchor$, there are two anchored versions of $\drawPtwo$. The quotient that identifies vertices $2$ and $4$ has the anchor at the end of the path, while the quotient that identifies the $\anchor$ and $3$ has its anchor on the middle vertex. $\exend$
\end{example}

With this definition of $\spasm^\anchor$ we can now show that homomorphism counts are indeed also suited to node-level tasks and the previous arguments apply similarly in this setting.
\begin{theorem}
\label{thm:local}
Let $G$ be an anchored graph and let $H$ be a graph and $v \in V(H)$. Then
\[
\subs(G, H,v) = \sum_{F \in \spasm^\anchor(G)} \alpha_F \cdot \homs(F, H)[\anchor \mapsto v]
\]
where $\homs(F, H)[\anchor \mapsto v]$ is the number of homomorphisms that map the anchor of $F$ to $v \in V(H)$.
\end{theorem}

It follows that the arguments of the previous sections apply analogously also to this setting. With respect to \Cref{sec:missing-piece} particularly,
\citet{LanzingerPB24} showed that $\homs(F, H)[\anchor \mapsto v]$ can be computed from the node-level $k$-WL labeling for every graph $F$ of treewidth $k$. Hence, in our example above, node-level homomorphism counts of $\drawCfour$ to 1-WL are sufficient additional information to determine node-level subgraph counts of the four cycle.

\subsection{It Goes Beyond Pattern Counting}
\label{sec:beyond}
We have mainly focused on the setting where a GNN is provided additional information motivated by some target structure that is considered important to the task (e.g., cycles).
However, the framework of graph motif parameters also motivates a natural alternative approach that is not dependent on such structure selection. 

For some (decidable) graph property $\phi$, let $\indsubs^\phi_k(G)$ be the function that maps $G$ to the number of induced subgraphs in $G$ with $k$ vertices that satisfy $\phi$. Through the choice of $\phi$  -- e.g., connectedness, hamiltonicity, planarity -- $\indsubs^\phi_k$ can cover an immense variety of graph parameters. For example,  $\indsubs(F,G) = \indsubs^\psi_{|V(F)|}(F)$ where $\psi$ is the property of being isomorphic to $F$ and  when $\phi$ is connectedness, $\indsubs^\phi_k(G)$ is the number of $k$-graphlets in $G$, a popular metric in graph analysis~(see \citet{RibeiroPSAS21}). 
It is known that all such functions $\indsubs^\phi_k$ are indeed graph motif parameters~\citep{RothS20}.

For some settings where there is no obvious kind of structure to inject information for, this presents an attractive alternative. In fact it is possible to provide practical amounts of additional information to a GNN, such that it can distinguish \emph{all} functions $\indsubs^\phi_k$ for small $k$.
 In concrete terms, by inspection of the arguments by~\citet{RothS20} in combination with \Cref{onlycon}, we observe the following.

\begin{proposition}
\label{prop:indsubs}
    Let $\gnn$ be a GNN, let $\phi$ be a decidable graph property, and let $k \geq 1$. 
    Let $\graphs^{\sf con}_{\leq k}$ be the set of connected graphs with at most $k$ vertices.
    Then $\indsubs^\phi_k$ can be expressed by $\gnn$ with $\homs(\graphs^{\mathit{con}}_{\leq k})$.
\end{proposition}

The possible applications of our approach range even further. Counting the satisfying assignments for certain fragments of first-order logic in a graph is also a graph motif parameter \citep{DellRW19}. This presents the possibility of targetedly injecting information to express  logic-specified properties.

\section{Experiments}
\label{sec:experiments}
We empirically validate our results on a variety of real-world tasks and architectures. Specifically, we conduct experiments for molecular property prediction with ZINC (\Cref{sec:gr-zinc}) and QM9 (\Cref{sec:qm9}), and for link prediction with COLLAB (\Cref{sec:collab}). We complement our empirical analysis with a dedicated expressiveness experiment on the BREC benchmark~(\Cref{sec:brec}). Full experimental details
are reported in \Cref{app:exp-details}.

\subsection{Graph Regression Experiments on ZINC}
\label{sec:gr-zinc}

\begin{table*}
\caption{We report the mean absolute error (MAE) for graph regression on the ZINC-12K dataset (without edge features). Using the $\spasm$ and $\spasm^{\anchor}$ counts yields very substantial improvements in performance of every model compared to their respective baselines and other approaches using $C_3,...,C_8$.}
\label{tab:basis-v-homiso.column}
\begin{center}
\begin{small}
\begin{sc}
\begin{tabular}{lcccc} 
\toprule
& GAT   & GCN  & GIN & BasePlanE \\ 
\midrule
Base    & 0.457\vartxtm{0.004}  & 0.417\vartxtm{0.007}  & 0.294\vartxtm{0.012}  & 0.124\vartxtm{0.004}  \\ 
\midrule
$\subs$     & 0.210\vartxtm{0.006}  & 0.206\vartxtm{0.006}  & 0.147\vartxtm{0.006}  & 0.108\vartxtm{0.002}\\ 

$\homs$     & 0.269\vartxtm{0.033}  & 0.254\vartxtm{0.017}  & 0.208\vartxtm{0.025}  & 0.106\vartxtm{0.004}\\ 
\midrule
$\spasm$& 0.155\vartxtm{0.006}  & 0.166\vartxtm{0.003}  & 0.158\vartxtm{0.004}  & 0.104\vartxtm{0.005}\\ 

$\spasm^\anchor$    & \textbf{0.147}\vartxtm{0.004} & \textbf{0.165}\vartxtm{0.004}& \textbf{0.146}\vartxtm{0.005}  & 0.100\vartxtm{0.002}  \\
\bottomrule
\end{tabular}
\end{sc}
\end{small}
\end{center}
\end{table*} 
\paragraph{Experimental Setup.} ZINC \citep{irwin2012zinc, dwivedi2023benchmarking} is a graph regression task for predicting the constrained solubility of molecules. Following \citet{dwivedi2023benchmarking}, we use a subset of the dataset that contains 12,000 graphs. Each graph represents an individual molecule, and the original node features denote the atom type of a given node.
We select GAT, GCN, and GIN, using the same experimental protocol and hyperparameters from \citet{dwivedi2023benchmarking}. We also select BasePlanE \cite{Dimitrov2023}, a recent architecture that is complete over planar graphs. For all models, we do not perform any additional model tuning. 

Our experiment compares the following configurations:
\begin{itemize}[leftmargin=20pt, noitemsep,topsep=0pt]
\item {\sc Base}: The base performance of the GAT, GCN, GIN, and BasePlanE model architectures.
\item {\subs}: The baseline models enriched with \emph{subgraph} counts of cycles $\{C_3,...,C_8 \}$.
\item {\homs}: The baseline models enriched with \emph{homomorphism} counts of cycles $\{C_3,...,C_8 \}$.
\item {\spasm}: The baseline models enriched with homomorphism counts of $\spasm(C_7) \cup \spasm(C_8)$\footnote{Note that $\spasm(C_7) \cup \spasm(C_8)$ also contain the basis graphs for $C_3,C_4,C_5$, and $C_6$ (recall \Cref{freesubgraphs}).}.
\item {$\spasm^\anchor$}: The baseline models enriched with homomorphism counts of $\spasm^\anchor(C_7) \cup \spasm^\anchor(C_8)$.
\end{itemize}
In all cases we inject the counts at node level, making $\subs$ and $\homs$ correspond directly to the approaches of \citet{BouritsasFZB23} and \citet{BarceloGRR21}, respectively. This allows us to precisely compare the effect of the homomorphism basis to these previous works. %
For $\spasm$ and $\spasm^{\anchor}$, we can obtain the global and local subgraph counts, respectively, so we include those counts as well. 
In all cases, the raw counts are encoded using a 2-layer MLP. Following the analysis of \Cref{sec:missing-piece}, we exclude counts for the acyclic graphs of the basis for GIN and BasePlanE experiments but add them for GAT and GCN.

\textbf{Results.} 
Our results are summarized in \Cref{tab:basis-v-homiso.column}.
In all models, adding the $\spasm^{\anchor}$ counts as additional node features yields the best performance. Despite ZINC being a graph-level task, we observe that precise node-level subgraph information is critical for top performance as $\spasm^\anchor$ outperforms $\spasm$. Additionally, 
in GAT and GCN, we clearly see the improvements in expressiveness from \Cref{thm.moreexpr} in effect as \spasm already clearly outperforms \textsc{Sub}.
This is further confirmed by the effectively equivalent performance of GAT/$\spasm^\anchor$, GIN/$\spasm^\anchor$, and GIN/\textsc{Sub}. That is, their performance is intuitively the performance of 1-WL together with node-level subgraph counts (up to $C_8$). Any additional expressiveness provided by $\spasm^\anchor$ is seemingly not relevant to the ZINC task. Since GIN is as expressive as $1$-WL, little is gained, whereas GAT manages to match GIN  through the increased expressiveness from the homomorphism basis.
Moreover, the gap between $\spasm$ and $\spasm^\anchor$  illustrates the discussion in \Cref{sec:node.level} and in particular the effect of $\spasm^\anchor$ as predicted in \Cref{thm:local}.

For BasePlanE, we observe that using the homomorphism basis instead of the parameter itself results in improved performance. This is insightful: even though BasePlanE can distinguish planar graphs (and all ZINC graphs are planar), distinguishing graphs is a weaker notion than capturing graph motif parameters, so injecting these parameters leads to empirical improvements even for stronger models specifically designed for planar graphs ( see also \Cref{app:zinc}).

\textbf{Using the Entire Basis.} 
To study the effect of the homomorphism basis in more detail, we perform experiments where we inject increasingly more (in order of the number of vertices, lowest  first) elements of the basis as features. To avoid the effect of the complex interplay of odd/even cycles on the benchmark, we use only $\spasm(C_8)$. \Cref{fig:spasm-gradient} clearly shows that the performance of every model improves as more of the basis is included. Moreover, just as in \Cref{tab:basis-v-homiso.column}, GAT and GIN converge in performance, again confirming that their expressiveness (as relevant to this task)  matches with this additional information.
\begin{figure}[hb]
	\centering
	\includegraphics[width=0.8\linewidth]{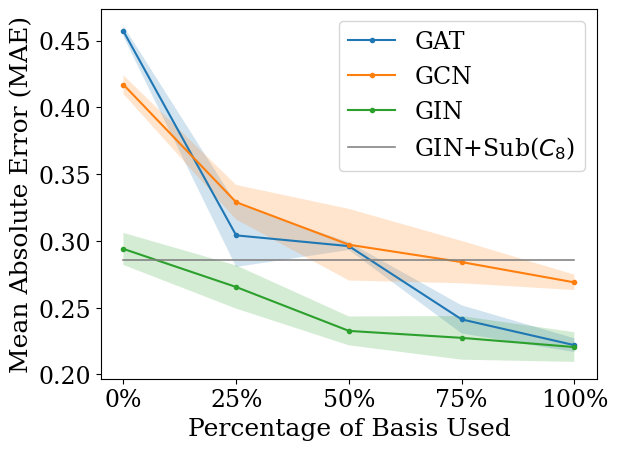}
	\caption{We report the effect of using incrementally more of $\spasm(C_8)$  as node features for each model. The performance of GIN with $\subs(C_8, \cdot)$ at node-level is also included for reference.}
	\label{fig:spasm-gradient}
\end{figure}

\subsection{Graph Regression Experiments on QM9}
\label{sec:qm9}

\textbf{Experimental Setup.}
QM9 is real-world molecular dataset consisting of 130,000 graphs \citep{wu2018moleculenet, Brockschmidt20} for graph regression over 13 different target properties. We select R-GCN \cite{schlichtkrull2018modeling} as our base model due to its comparatively stable behavior in the  original benchmark \citep{Brockschmidt20}. %
Contrasting \Cref{sec:gr-zinc}, we inject only graph-level homomorphism counts (rather than node-level). The additional features are injected only once as additional inputs to the final MLP. The regression tasks in QM9 target a variety of quantum chemical properties of the molecule and it is unclear which specific pattern information might be helpful. We thus follow the ideas from \Cref{sec:beyond} and inject counts for $\graphs^{\sf con}_{\leq 5}$ (connected graphs with at most $5$ vertices) and the 6-cycle $C_6$. This set contains all of $\spasm(C_6)$. Baseline results for R-GCN are from \citet{Brockschmidt20}.  %

\begin{table}[t]
\caption{We report the mean absolute error (MAE) for graph regression on QM9 for all 13 proprieties. Using the homomorphism counts yields significant improvements over the baseline model.}
\label{tab:qm9.smol}
\begin{center}
\begin{small}
\begin{sc}
\begin{tabular}{lrr} 
\toprule
Property & R-GCN                & + $\homs(\graphs^{\sf con}_{\leq 5} \cup \{ C_6\})$    \\ 
\midrule
mu       & 3.21 \vartxtm{0.06}  & \textbf{2.29 \vartxtm{0.03}}       \\
alpha    & 4.22 \vartxtm{0.45}  & \textbf{1.77 \vartxtm{0.05}}       \\
HOMO     & 1.45 \vartxtm{0.01}  & \textbf{1.30 \vartxtm{0.03}}       \\
LUMO     & 1.62 \vartxtm{0.04}  & \textbf{1.41 \vartxtm{0.02}}       \\
gap      & 2.42 \vartxtm{0.14}  & \textbf{2.00 \vartxtm{0.04}}      \\
R2       & 16.38 \vartxtm{0.49} & \textbf{10.29 \vartxtm{0.35}}      \\
ZPVE     & 17.40 \vartxtm{3.56} & \textbf{3.03 \vartxtm{0.38}}      \\
U0       & 7.82 \vartxtm{0.80}  & \textbf{1.09 \vartxtm{0.18}}       \\
U        & 8.24 \vartxtm{1.25}  & \textbf{1.21 \vartxtm{0.17}}       \\
H        & 9.05 \vartxtm{1.21}  & \textbf{1.22 \vartxtm{0.14}}       \\
G        & 7.00 \vartxtm{1.51}  & \textbf{1.14 \vartxtm{0.13}}      \\
Cv       & 3.93 \vartxtm{0.48}  & \textbf{1.46 \vartxtm{0.08}}       \\
Omega    & 1.02 \vartxtm{0.05}  & \textbf{0.81 \vartxtm{0.02}}       \\
\bottomrule
\end{tabular}
\end{sc}
\end{small}
\end{center}
\vskip -0.3in
\end{table}

\textbf{Results.}
Table \ref{tab:qm9.smol} shows that using the homomorphism counts for $ \graphs^{\sf con}_{\leq 5}$ and $C_6$ significantly improves the performance of the base R-GCN model on all properties, even without additional model tuning. 
\citet{Alon-ICLR21} showed that the performance of R-GCN on this task can also be boosted with a fully-adjacent (FA) layer. We report performance increases also in combination with FA in \Cref{app:qm9} (see~\Cref{tab:qm9}).
These observations demonstrate the well-foundedness of our approach as it confirms the practical relevance of \Cref{sec:beyond} and \Cref{prop:indsubs}.
Moreover, we see that the addition of graph level homomorphism counts alone are sufficient for substantial performance improvements. This is particularly noteworthy as it simplifies integration with more complex models, where node-level injection of many features might affect hyperparameters significantly. To the best of our knowledge, this is also the first study demonstrating the potential of injecting graph-level homomorphism information to a GNN.

\subsection{Link Prediction on COLLAB}
\label{sec:collab}
So far, we have demonstrated multiple ways in which our approach is effective on graph-level regression tasks on small graphs. In this experiment, we contrast this by performing a link prediction task on a graph where direct subgraph counting is intractable for most patterns.

\textbf{Experimental Setup.}
We evaluate the link prediction task for the COLLAB dataset from Open Graph Benchmark \citep{OGB-NeurIPS2020}. The dataset contains a single undirected graph with over 235,000 nodes  -- each with 128 features -- and roughly 1.3 million edges, which represents a collaboration network. Previous works \citep{BarceloGRR21} enriched the node features by the number of cliques (i.e., $K_3, K_4, K_5$) the node is part of. We select GAT and GCN as our baseline models and enrich their node features with the node-level homomorphism counts of the $\spasm$ of different-sized cliques and paths. The \spasm of a clique contains only itself, and the clique experiments thus match the experiments conducted by \citet{BarceloGRR21}. 
To control for high counts values that occur in large graphs, we encode them via positional encoding  \cite{vaswani2017attention}. We provide a more in-depth analysis of the importance of using the positional encoding later on.

{\setlength{\tabcolsep}{5pt}
\begin{table}[t]
\caption{\small We report the Hits@50 metric for link prediction on COLLAB. We use the homomorphism basis of $n$-vertex paths $P_n$ and cliques $K_n$. The first column shows the additional graphs in the \spasm of each pattern. Since $\spasm(P_\ell) \subseteq \spasm(P_{\ell+1})$ we show only the additional basis elements in the respective rows.}
\label{tab:collab-paths}
\begin{center}
\begin{small}
\begin{sc}
\begin{tabular}{clccc}
\toprule
add. basis & Pattern & GCN &  GAT \\
\midrule
& {\scriptsize ---}       & 46.13\%\tiny{$\pm 2.10$} & 48.27\%\vartxt{1.05} \\
\midrule

& $K_3$           & 49.41\%\tiny{$\pm 0.42$} &  49.43\%\vartxt{0.73} \\

& $K_4$           & 47.76\%\tiny{$\pm 0.53$} &  48.35\%\vartxt{0.85}\\

& $K_5$           & 48.01\%\tiny{$\pm 0.87$} & 49.75\%\vartxt{0.16}\\

\midrule
\drawPone \drawPtwo \!\! \drawKthree& $P_4$   & 49.59\%\tiny{$\pm 0.23$} & 50.76\%\vartxt{0.51}\\

\midrule
\drawStarthree \drawKthreeplusone \drawCfour & $P_5$   & 49.60\%\tiny{$\pm 0.29$}  & 51.55\%\vartxt{0.94} \\
\midrule

    \drawStarthreePlusOne \, \drawAlmostKfour  \drawKthreePlusTwo  & \multirow{2}{*}{$P_6$}  & \multirow{2}{*}{ \textbf{50.35\%\tiny{$\pm 0.21$}} }& 
\multirow{2}{*}{ {\bf 51.62\%}\vartxt{0.66}}\\

\drawKthreeplusPtwo \, \drawCfourPlusOne \drawCfive 

 \\
\bottomrule
\end{tabular}
\end{sc}
\end{small}
\end{center}
\vskip -0.2in
\end{table}
}

\textbf{Results.}
From \Cref{tab:collab-paths} we see that our approach of injecting the basis homomorphism counts for paths outperforms the previously studied injection of clique counts in both models. This further demonstrates the targeted gain of expressiveness, irrespective of task, that can be realized through our approach.
We note that while short paths are seemingly simple structures, actually counting them is not expressible in 1-WL as illustrated by the cyclic graphs in the basis (and thus also in the tested, weaker models).
Despite the large graph size, we compute all necessary node-level homomorphisms counts in roughly 3 hours (see \Cref{app:runtimes}).

\textbf{Ablation with Positional Encoding.}
To study the impact of normalizing homomorphims counts with a sinusoidal positional encoding \cite{vaswani2017attention}, we perform an ablation with GAT on the COLLAB dataset. We compare the performance of using an MLP encoder to using a positional encoder for the $\homs(\spasm(P_6))$ features. From \Cref{fig:ablation}, we see that including counts via an MLP encoder actually degrades performance of the baseline model and introduces high variance. This is likely because homomorphism count values can grow extremely large in larger graphs, causing the model to treat the extra information as noise. Conversely, using a positional encoding allows us to normalize the raw homomorphism count values in a way that preserves the relative distances between different patterns across different graphs. However, adjusting the positional encoding dimension presented minimal effects in our ablation.

\begin{figure}[t]
	\centering
	\includegraphics[width=0.85\linewidth]{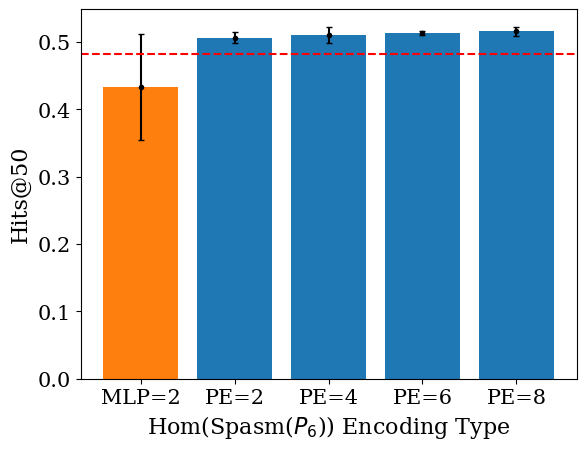}
	\caption{We report the effect of using different encoding methods and dimensions for the $\spasm(P_6)$ homomorphism count features with GAT on COLLAB. The dashed red line indicates the performance of GAT without additional features. Using a positional encoding (PE) boosts performance of the baseline model, whereas a standard MLP encoder degrades  performance. }
    \label{fig:ablation}
 \vspace{-1em}
\end{figure}

\subsection{Expressiveness Evaluation on BREC}
\label{sec:brec}
BREC is a recent expressiveness benchmark containing 400 pairs of non-isomorphic graphs \citep{wang2023towards}. The benchmark tests how capable models are at distinguishing the graphs in these pairs. The pairs are split into four categories: Basic, Regular, Extension and CFI graphs. 

\textbf{Experimental Setup.}
We select GIN and PPGN \citep{MaronBSL19} as the two models to extend with node-level homomorphism counts. We again opt against choosing a particular pattern and instead inject the homomorphism counts of $\graphs^{\sf con}_{\leq 5}$ as node features in both cases. For GIN + $\homs(\graphs^{\sf con}_{\leq 5})$, we use the same experimental setup as was used for GSN in \citep{wang2023towards}, replacing subgraph counts directly with homomorphism counts. For PPGN + $\homs(\graphs^{\sf con}_{\leq 5})$ we normalize features in accordance with \citet{MaronBSL19}. 

\textbf{Results.}
A summary of results is presented in Table \ref{tab:small-brec}. It compares our results to the previous state of the art I$^2$-GNN, and the closely related models PPGN and GSN.
We achieve state of the art results using PPGN + $\homs(\graphs^{\sf con}_{\leq 5})$. Notably, in the category of regular graphs, where PPGN was initially lacking, the inclusion of our approach enables the model to distinguish over twice as many graphs. This experimentally illustrates our analysis in \Cref{sec:missing-piece} in the context of higher-order models: $\graphs^{\sf con}_{\leq 5}$  contains graphs with treewidth greater than $2$ (in particular 4 graphs containing a $4$-clique). These homomorphism counts thus improve the 2-WL\footnote{Recall that we always refer to the folklore version of $k$-WL.} expressiveness of PPGN. 
GIN + $\homs(\graphs^{\sf con}_{\leq 5})$ also performs well by ranking 4th overall, outperforming many higher-order models (see Appendix \ref{app:brec} for full results).

{\setlength{\tabcolsep}{3.5pt}
\begin{table}[t]
\caption{Summary of BREC expressiveness experiments. Percentages refer to the percentage of graph pairs in each category that were correctly distinguished by the model.}
\label{tab:small-brec}
\begin{center}
\begin{small}
\begin{sc}
\begin{tabular}{lcccc|c}
\toprule

                 & Bas.  & Reg. & Ext. & CFI & Total \\
\toprule
2-WL$^4$             & 100\%      & 35.7\%        & 100\%           & 60\%      & 67.5\%      \\
\midrule
PPGN            & 100\%      & 35.7\%        & 100\%           & 23\%      & 58.2\%      \\
GSN              & 100\%      & 70.7\%        & 95\%            & 0\%       & 63.5\%      \\
$\text{I}^2$-GNN & 100\%      & 71.4\%        & 100\%           & 21\%      & 70.2\%      \\
\midrule
GIN+$\homs(\graphs^{\sf con}_{\leq 5})$ & 100\%      & 85.7\%        & 96\%            & 0\%       & 69\%        \\
PPGN+$\homs(\graphs_{\leq 5}^{\sf con})$ & 100\%      & 85.7\%        & 100\%           & 25\%      & \bf 76.25\%     \\
\bottomrule
\end{tabular}
\end{sc}
\end{small}
\end{center}
\vspace{-2em}
\end{table}
} 

\section{Discussion and Future Work}
We have proposed the injection of the homomorphism basis of graph parameters into GNNs as a flexible framework for increasing expressiveness in a targeted and efficient manner. We experimentally validate that this increase in expressiveness is reflected in significant performance gains in real-world GNN tasks.
This lays the foundation for a wide-range of  further applications: 
both node- and graph-level homomorphism counts have proven to boost performance in a variety of tasks and architectures.  Our experiments have primarily focused on standard architectures but motivate combination with complex top performing models to further boost their performance.
Moreover, the use of homomorphism bases induced by logical formulas (\Cref{sec:beyond}) has been beyond the scope of this paper but presents numerous interesting possibilities.

The broad applicability of our approach brings with it a number of technical question that require more detailed analysis. In large graphs, homomorphism counts can span a large range of values, opening up questions of how to best encode them to realize the theoretical expressiveness gain. Additionally, the choice of basis to inject depends on the desired expressivity and further work is needed to explore the effects of different bases in common GNN tasks.
 
\section*{Acknowledgements}
Emily Jin is partially funded by AstraZeneca and the UKRI Engineering and Physical Sciences Research Council (EPSRC) with grant code EP/S024093/1. Matthias Lanzinger acknowledges support by the Royal Society ``RAISON DATA" project (Reference No. RP\textbackslash{}R1\textbackslash{}201074) and by the Vienna Science and Technology Fund (WWTF) [10.47379/ICT2201]. This work was also partially funded by EPSRC Turing AI World-Leading Research Fellowship No. EP/X040062/1.

\section*{Impact Statement}
This paper presents work whose goal is to advance the field of Machine Learning. There are many potential societal consequences of our work, none which we feel must be specifically highlighted here.

\bibliography{lib}
\bibliographystyle{icml2024}

\newpage
\appendix
\onecolumn

%

\section{Technical Details}
\label{app:proofs}

\begin{proof}[Proof of \Cref{freesubgraphs}]
    Since $\subs(G, \cdot)$ is a graph motif parameter with support equal to $\spasm(G)$, we have that 
    $\subs(G, \cdot)$ is determined by $\homs(\spasm(G))$ since it functionally depends on the terms of $\homs(\spasm(G))$.
    
    Observe that if $F \in \spasm(G)$, then $\spasm(F) \subseteq \spasm(G)$. In particular, $H \in \spasm(F)$ implies that $H \isomorphic F/\rho$ for some partition $\rho$ of $V(F)$. Furthermore, $F \isomorphic G/\rho'$ where $\rho'$ is some partition of $V(G)$ since $F \in \spasm(G)$. Thus, $\rho$ is a partition of blocks of $\rho'$, and hence there is also a partition $\rho''$, obtained by applying the two partitions one after another, such that $H \isomorphic G/\rho''$.
\end{proof}

\begin{proof}[Proof of \Cref{prop:indsubs}]
    Investigation of the proof of \citet[Theorem 12]{RothS20} reveals that $\supp(\indsubs^\phi_k) \subseteq \graphs_k$.
    Thus, in relation to our setting this means that a GNN $\gnn$ with $\homs(\graphs_k)$ expresses $\indsubs^\phi_k$. We then apply \Cref{onlycon} to observe the final statement.
\end{proof}

\subsection{\Cref{thm.moreexpr}}
\label{app:moreexpr}

We will prove a more general version of the theorem in the main body that holds not just for the steps of the $k$-WL hierarchy, but for all GNN models that follow a certain natural behaviour with respect to homomorphisms. 
It is well known that $G \equiv_{k\mathrm{-WL}} H$ if and only if $\homs(F,G)=\homs(F,H)$ for all graphs $F$ of treewidth at most $k$ \citep{Dvorak10}. Furthermore, for every graph $F$ with treewidth greater than $k$, there are $G,H$ with $G \equiv_{k\mathrm{-WL}} H$ where $\homs(F,G)\neq \homs(F,H)$, i.e., the class of graphs with treewidth at most $k$ is the maximal class for which homomorphism vector equivalence is equivalent to $k$-WL equivalence \citep{neuen2023homomorphism}. This means that there is a concrete set of graphs whose homomorphism counts precisely determine the expressivity of the model. \citet{zhang2024beyond} recently showed that this property also holds true for other commonly studied models that do not precisely match the expressiveness of the $k$-WL hierarchy, e.g., Subgraph GNNs and Local $k$-GNNs. 

\begin{definition}
    Let $G,H$ be graphs and let $\mathcal{F}$ be a set of graphs. We say that $G$ and $H$ are \emph{homomorphism indistinguishable} under $\mathcal{F}$, and we write $G \equiv_{\mathcal{F}} H$, if 
    \(
    \homs(F,G) = \homs(F,H) 
    \) for all $F \in \mathcal{F}$.
\end{definition}

\begin{definition}
    We say that a class of graphs $\mathcal{F}$ is \emph{closed under restriction to connected components} if for every graph $F \in \mathcal{F}$, every \emph{maximal} connected component of $F$ is in also in $\mathcal{F}$.
\end{definition}

\begin{definition}
    Let $\gnn$ be a GNN model. We say that $\gnn$ is \emph{strongly homomorphism expressive} if there is a set of graphs $\mathcal{F}$ such that:
    \begin{enumerate}[label=(\roman*)]
        \item for every pair of graphs $G,H$, it holds that $G \equiv_{\gnn} H$ if and only if $G \equiv_{\mathcal{F}} H$,
        \item $\mathcal{F}$ is closed under restriction to connected components,
        \item and $\mathcal{F}$ is the maximal set satisfying the first property.
    \end{enumerate}
    We refer to $\mathcal{F}$ as the  homomorphism expressiveness of $\gnn$.
\end{definition}

As mentioned above, any GNN model with expressiveness equal to $k$-WL is strongly homomorphism expressive with $\mathcal{F}$ being the set of treewidth $k$ graphs. \citet{zhang2024beyond} showed that Local $k$-GNNs, local $k$-FGNNs, and Subgraph GNNs are also strongly homomorphism expressive. 

In this light it also becomes clear when
the expressiveness of strongly homomorphism expressive $\gnn$ with $\homs(\spasm(G))$ is strictly more expressive than $\gnn$ with $\homs(G, \cdot)$: exactly when there is some $F\in \spasm(G)$ that is not in the homomorphism expressiveness of $\gnn$.

The theorem we will prove in the rest of this section is the following stronger version of \Cref{thm.moreexpr}.

\begin{theorem}
\label{thm.app.expr}
Let $\Gamma$ be a connected graph motif parameter and let $\gnn$ be a strongly homomorphism expressive GNN model. 
    Then $\gnn$ with $\extrabasis_\Gamma$ is at least as expressive as $\gnn$ with $\{\Gamma\}$.
    Moreover, if at least two functions in $\homs(\mathsf{Supp}(\Gamma))$ are not expressed by $\gnn$ with $\{\Gamma\}$, then $\gnn$ with $\extrabasis_\Gamma$ is strictly more expressive.
\end{theorem}

Observe that the condition for the strictly more expressive case cannot be improved. If only one term $\homs(F^*,\cdot)$ is not expressed by $\gnn$, then
for every $G,H$ with $\Gamma(G)=\Gamma(H)$ and $G \equiv_\gnn H$ we have from the former that
\[
    \sum_{F \in \supp(\Gamma)} \alpha(F) \homs(F,G) = \sum_{F \in \supp(\Gamma)} \alpha(F) \homs(F,H).
\]
Since all terms in the sum, except for $F^*$ are expressed by $\gnn$, those terms must be equal on both sides. Hence we are left with $\homs(F^*,G)=\homs(F^*,H)$ whenever $\Gamma(G)=\Gamma(H)$.

The technical challenge in proving \Cref{thm.app.expr} comes from the analysis of how equality under $\gnn$ interacts with equality under $\Gamma$. The definition of strongly homomorphism expressive GNNs guarantees, that for every graph $F$ for which $\homs(F,\cdot)$ is not expressed by $\gnn$, there are graphs $G \equiv_\gnn H$ that differ on $\homs(F,\cdot)$. On a technical level, the existence of such graphs needs to imply that there are graphs that are equivalent under both $\gnn$ and $\Gamma$. Our proof shows how to construct such a graph whenever $\Gamma$ is connected. Our argument requires connectedness of the graph motif parameter in order to control $\Gamma$ in the construction. In particular, when $\Gamma$ is connected, it is additive over disjoint union, which ultimately allows us to construct graphs through these operations in a way where $\Gamma$ is controlled.

We will first observe a number important properties of strongly homomorphism expressive GNNs based on well known properties of $\homs$. To this end we first define two operations on graphs. For graphs $G,H$ we write $G + H$ for the disjoint union of $G$ and $H$, i.e., the graph that is obtained by relabeling $H$ to share no vertices with $G$ and then combining their vertices and edges into a single graph.
Moreover, we write $G \times H$ for their \emph{categorical product}, which is the graph $P$ with 
$V(P) = V(G) \times V(H)$ and $((a,u), (b,v))\in E(P)$ if and only if $(a,b) \in E(G)$ and $(u,v) \in E(H)$.
For positive integer $c$ we write $cG$ for the $c$-fold disjoint union of $G$ with itself, and similarly $G^c$ for the $c$-fold product.

\begin{proposition}[\citet{lovasz1967operations}]
\label{lem:hombasics}
For graphs $F,G,H$, the following statements hold:
\begin{itemize}
    \item $\homs(F+G, H) = \homs(F,H)\homs(G,H)$,
    \item $\homs(F, G+H) = \homs(F,G)+\homs(F,H)$ if $F$ is connected,
    \item $\homs(F, G\times H) = \homs(F,G)\homs(F,H)$.
\end{itemize}
\end{proposition}

From the first point of \Cref{lem:hombasics} it follows that for any disconnected graph $G$, the homomorphisms from $G$ into any other graph are determined fully by the homomorphism counts from the connected components of $G$. \Cref{onlycon} then follows immediately.

\begin{lemma}
\label{lem:gnn.algebra.closed}
    Let $\gnn$ be a strongly homomorphism expressive GNN. Let $G,G',H,H'$ be graphs such that $G \equiv_\gnn G'$ and $H \equiv_\gnn H'$.
    The following two statements hold:
    \begin{enumerate}
        \item $G + H \equiv_\gnn G' + H'$
        \item $G \times H \equiv_\gnn G' \times H'$
    \end{enumerate}
\end{lemma}
\begin{proof}
Let $\mathcal{F}$ be the homomorphism expressiveness of $\gnn$.
    Let $F$ be a connected graph in $\mathcal{F}$. We have 
    \[  
    \homs(F, G+H) = \homs(F,G)+\homs(F,H) = \homs(F,G')+\homs(F,H') = \homs(F,G'+H')
    \]
    where the outer equalities are by \Cref{lem:hombasics} and the middle equality is by the assumption that $G \equiv_\gnn G'$ and $H \equiv_\gnn H'$. The argument for the second statement is analogous.

    This completes the argument for connected $F$.
    Since $\mathcal{F}$ is closed under restriction to
    connected components, every graph $F' \in \mathcal{F}$ is a disjoint union of connected graphs in $F$. By the first point of \Cref{lem:hombasics}, the equality on all the connected components of $F'$ immediately extends to $F'$.
\end{proof}

  \begin{proof}[Proof of \Cref{thm.app.expr}]
    Let $A,B$ be two graphs in $\supp(\Gamma)$ that are not expressed by $\gnn$ with $\Gamma$. Our plan will be to show that there are graphs $G,H$ that are equivalent under $\gnn$ and have $\Gamma(G)=\Gamma(H)$, but can be distinguished by either $\homs(A,\cdot)$ or $\homs(B,\cdot)$.

  \begin{claim} Assuming $\gnn$ with $\Gamma$ does not express at least one of $\homs(A,\cdot)$ or $\homs(B,\cdot)$, then there are graphs $G_A, H_A, G_B, H_B$ that satisfy the following properties:
    \begin{itemize}[topsep=2pt]
        \item $G_A \equiv_\gnn H_A$ and $G_B \equiv_\gnn H_B$ 
        \item $\homs(A,G_A) \neq  \homs(A,H_A)$ and $\homs(B, G_B) \neq \homs(B, H_B)$ 
       \item $\max(\homs(A,G_A),  \homs(A,H_A)) > \max(\homs(B, G_A), \homs(B, H_A))$%
    \end{itemize}
    \end{claim}
    \textit{Proof of claim:}
        The first two items follow immediately by assumption that $\gnn$ cannot express $\homs(A,\cdot)$ and $\homs(B, \cdot)$.
        For the last point, assume the two maxima are the same, say $\homs(A,G_A)=\homs(B,G_A) > \homs(B,H_A)$. Recall that it is known that there always exists a graph $X$ on which $\homs(A,X) \neq \homs(B,X)$ (\citep{lovasz1967operations}). Assume $\homs(A,X)>\homs(B,X)$ without loss of generality.  We have $G_A + X \equiv_\gnn H_A+X$ by \Cref{lem:gnn.algebra.closed}. That is we have $\homs(A,G_A+X) = \homs(A,G_A)+\homs(A,X)> \homs(B,G_A) + \homs(B,X) = \homs(B,G_A+X)$. Similarly, $\homs(A,H_A+X) = \homs(A,H_A)+\homs(A,X) < \homs(A,G_A)+\homs(A,X)$. So replacing $G_A,H_A$ with $G_A+X,H_A+X$ satisfies the statement.
        \hfill \claimqed

    If $\Gamma(G_A)= \Gamma(H_A)$ or $\Gamma(G_B)= \Gamma(H_B)$, the theorem would be proven: if $\Gamma$ would agree on one of the pairs, $\gnn$ with $\Gamma$ would not express the respective homomorphism count function.
    Assuming that $\Gamma$ is indeed the same on these two pairs of graphs, we show that we can always construct graphs $G,H$ that are still equivalent under $\gnn$ and disagree on either $\homs(A,\cdot)$ or $\homs(B, \cdot)$, but also have $\Gamma(G)=\Gamma(H)$.
    We make a key observation on $\Gamma$ first. Recall that $\Gamma$ is a \emph{connected} graph motif parameter, thus from \Cref{lem:hombasics} together with \Cref{lgmp}, we observe that $\Gamma(G+H)=\Gamma(G)+\Gamma(H)$. 

    For positive integer $i$, we define $\delta_i = \Gamma(G_A^i) - \Gamma(H_A^i)$ and similarly $\eta_i = \Gamma(H_B^i) - \Gamma(G_B^i)$ (recall, $G^i$ is the $i$-fold categorical product). We will only care about $\eta_0$, which we therefore refer to as simply $\eta$. Assume, w.l.o.g., $\eta > 0$, otherwise switch $G_B,H_B$ accordingly.
    
    For $i$ where $\delta_i > 0$ define $G_i := \eta G_A^i + \delta_i G_B$  an $H_i := \eta H_A^i + \delta_i H_B$. 
    \begin{claim}
        For every positive integer $i$ where $\delta_i > 0$ we have $\Gamma(G_i) = \Gamma(H_i)$
        and $G_i \equiv_\gnn H_i$.
    \end{claim}
    \textit{Proof of claim:}
    Recall from above that $\Gamma(cG) = c\Gamma(G)$.
    For the equality under $\Gamma$ we then observe that
    \begin{align*}
    & \Gamma(G_i) - \Gamma(H_i) =\\ 
    & \eta \Gamma(G_A^i) + \delta_i \Gamma(G_B) - \eta \Gamma(H_A^i) - \delta_i \Gamma(H_B) = \\ 
    & \delta_i (\Gamma(G_B) - \Gamma(H_B)) + \eta (\Gamma(G_A^i) - \Gamma(H_A^i))  =\\
    & \delta_i (-\eta) + \delta_i\eta= 0
    \end{align*}
    Since $G_i$ and $H_i$ are constructed from disjoint unions and products of graphs that are equivalent under $\gnn$, we also have $G_i \equiv_\gnn H_i$
    by \Cref{lem:gnn.algebra.closed}. \hfill \claimqed

    There are arbitrarily many $i$ such that $\delta_i$ has positive sign (if the sign is negative for all $\delta_i$, switch $G_A,H_A$). As above, we are done if $\homs(A, G_i) \neq \homs(A,H_i)$ for any $i$ where $\delta_i>0$ since we have $G_i\equiv_\gnn H_i$ and $\Gamma(G_i)=\Gamma(H_i)$ (and the same for $B$). Let us thus assume that this is not the case. For each of $i$ we then get the following equations:
    \begin{align}
        \eta \homs(A,G_A^i) + \delta_i \homs(A, G_B) = \homs(A,G_i) = \homs(A,H_i)= \eta \homs(A,H_A^i) + \delta_i \homs(A, H_B)  \\
         \eta \homs(B,G_A^i) + \delta_i \homs(B, G_B) = \homs(B,G_i) = \homs(B,H_i)= \eta \homs(B,H_A^i) + \delta_i \homs(B, H_B)  \label{proofB}
    \end{align}
    First, suppose that $ \homs(B,G_A^i) = \homs(B,H_A^i)$. Simplifying \Cref{proofB} gives $\homs(B,G_B)=\homs(B,H_B)$ which contradicts our choice of $B, H_B,G_B$.
    In the other case we have $\homs(B,G_A^i) \neq \homs(B,H_A^i)$.
    Solving both equations for $\frac{\eta}{\delta_i}$ (recall $\delta_i \neq 0$), identifying and rewriting gives

    \begin{align}
    \label{eq:key}
      \frac{\homs(A, H_B) - \homs(A, G_B)}{\homs(B, H_B) - \homs(B, G_B)}
     = 
     \frac{\homs(A,G_A^i) - \homs(A,H_A^i)}{ \homs(B,G_A^i) - \homs(B,H_A^i)}
     \end{align}
    
     Observe that the left side of \Cref{eq:key} is independent of $i$, and in fact is the same for arbitrarily other integers $j$.
     That is, the ratio on the right must be equal regardless of choice of integers. Recalling that $\homs(F, G^c) = \homs(F,G)^c$ from \Cref{lem:hombasics} we see
     \begin{align}
         \frac{\homs(A,G_A)^i - \homs(A,H_A)^i}{ \homs(B,G_A)^i - \homs(B,H_A)^i} = \frac{\homs(A,G_A)^j - \homs(A,H_A)^j}{ \homs(B,G_A)^j - \homs(B,H_A)^j}
     \end{align}
     which contradicts our initial choice of graphs where we had $\max(\homs(A,G_A),  \homs(A,H_A)) > \max(\homs(B, G_A), \homs(B, H_A))$. That is, the ratio on the right-hand side of \Cref{eq:key} cannot stay constant. The equation was derived from the assumption that $\homs(A,\cdot)$ or $\homs(B,\cdot)$ both (individually) are the same for $G_i$ and $H_i$. In consequence, there must be an $i$ for which $\Gamma(G_i)=\Gamma(H_i)$ and $G_i \equiv_\gnn H_i$ but $\homs(A, G_i)\neq \homs(A, H_i)$ or $\homs(B, G_i)\neq \homs(B, H_i)$. Say $\homs$ is different for $A$, then $\gnn$ with $\Gamma$ does not express $\homs(A,\cdot)$.%
\end{proof}

\subsection{\Cref{thm:local}}

In addition to the counting functions introduced in the main body, we will require the counts for two more special kinds of homomorphisms. To that end we write $\injs(G,H)$ for the number of \emph{injective} homomorphisms from $G$ to $H$. Moreover, an \emph{automorphism} of $G$ is a homomorphism from $G$ to $G$, and we write $\mathsf{Aut}(G)$ for the number of automorphisms of $G$. In particular, we will care about anchored graphs $G$ and the number of automorphisms that map the anchor to itself, which we denote as $\mathsf{Aut}^\anchor(G)$. 
We write  $\subs(G,H)[\anchor=v]$ for the number subgraphs of $G$ that are isomorphic to anchored graph $G$ when $v$ in $H$ is set as the anchor of $H$. This is equivalent to $\subs(G,H,v)$ as defined in the main body of the text but aligns better with the proof.

Finally, we will need to reason about the lattice of vertex partitions of a graph. Let $\rho,\rho'$ be partitions of $V(G)$. We say that $\rho$ is a coarsening of  $\rho'$, or $\rho \geq \rho'$ if for every block $B'\in \rho'$ there is a block $B \in \rho$ such that $B' \subseteq B$. 

\begin{proof}[Proof of \Cref{thm:local}]
     With the appropriate setup we can follow the same steps as the analogous graph level proofs by \citet{CurticapeanDM17}. Recall that we want to show that
    \[
    \subs(G,H)[\anchor=v] = \sum_{F \in\spasm^\anchor(G)} \alpha_F \cdot \homs(F,H)[\anchor \mapsto v]
    \]
    
    We start from two basic observations. First
    \begin{equation}
        \homs(G,H)[\anchor \mapsto v] = \sum_\rho \injs(G/\rho,H)[\anchor \mapsto v]
        \label{eq:zeta}
    \end{equation}
    where the sum ranges over all partitions of $V(G)$. The second is that
    \begin{equation}
        \injs(G,H)[\anchor \mapsto v] =  \mathsf{Aut}^\anchor(G) \subs(G,H)[\anchor=v]
    \label{eq:inj}
    \end{equation}
    where $\mathsf{Aut}^\anchor$ is the number of automorphisms of $G$ that map the anchor to itself.
    Observe that $\homs(G,H)[\anchor \mapsto v]$ is the upward zeta transformation of $\injs(G/\rho,H)[\anchor \mapsto v]$ on the partition lattice from the trivial partition $\bot$ where every block is a singleton.

    M\"obius transformation of \Cref{eq:zeta} then gives us the expression
    {
    \begin{equation}
        \injs(G/\rho,H)[\anchor \mapsto v] = \sum_{\rho'\geq \rho}  (-1)^{|\rho|-|\rho'|}\left(\prod_{B \in \rho'}(\lambda(\rho,\rho',B)-1)! \right) \homs(G/\rho',H)[\anchor \mapsto v]
    \end{equation}
    }
    which, fixing $\rho=\bot$ is
    \begin{equation}
        \injs(G,H)[\anchor \mapsto v] = \sum_{\rho'\geq \rho}  (-1)^{|V(G)|-|\rho'|}\left(\prod_{B \in \rho'}(|B|-1)! \right) \homs(G/\rho',H)[\anchor \mapsto v]
    \end{equation}
    
    Note that by definition, $G/\rho'$ is isomorphic to a term in $\spasm^\anchor(G)$ for every partition $\rho'$.
    Collecting the terms for every graph in $\spasm^\anchor(G)$ we then get
    {
    \begin{equation}
        \injs(G,H)[\anchor \mapsto v] = \sum_{F \in \spasm^\anchor(G)}  (-1)^{|V(G)|-|V(F)|} \homs(F,H)[\anchor \mapsto v] \left(\prod_{\substack{B \in \rho' \\ G/\rho' \isomorphic F}}(|B|-1)! \right)
    \end{equation}
    }
    Combining this with \Cref{eq:inj} we get desired expression for $\subs(G,H)[\anchor=v]$. In particular, the coefficient $\alpha_F$ is
    \begin{equation}
         \alpha_F  = \frac{(-1)^{|V(G)|-|V(F)|}}{\mathsf{Aut}^\anchor(G)} \left(\prod_{\substack{B \in \rho' 
         \\ G/\rho' \isomorphic F}}(|B|-1)! \right)
    \end{equation}
\end{proof}

\section{Regarding Complexity}
\label{app:complexity}

In this section we summarize key results from computational complexity theory that pertain to the argument of \Cref{sec:efficient}. For the sake of a simpler and more streamlined presentation we will first focus on the special case of subgraph counting. Afterwards, we discuss how these results apply in general for graph motif parameters and how this relates to practical algorithms.
To begin, we thus consider the following algorithmic problem.
\begin{problem}{\subprob}
  Input & Graphs $G$, $H$ \\
  Output & $\subs(G,H)$
\end{problem}

The problem is well known to be $\#\sf P$-hard (intuitively, the counting analogue of $\sf NP$-hardness) and therefore considered to generally be intractable. It is then natural to consider for which pattern graphs -- $G$ in the definition of \subprob above -- there might be specialized, more efficient algorithms. This particular question is primarily studied in the context of \emph{parameterized complexity theory} and in particular the problem parameterized by the pattern graph $G$, defined as follows.
\begin{problem}{p-\subprob}
  Input & Graphs $G$, $H$ \\
  Parameter & $G$ \\
  Output & $\subs(G,H)$
\end{problem}
In the parameterized setting, the classical requirement for tractability is relaxed. For graph $G$ we write $|G|$ for $|V(G)|+|E(G)|$\footnote{For the sake of complexity analysis this is the size of a standard representation of a graph in a Turing Machine up to a $\log$ factor.}. A \emph{fixed-parameter tractable} (or fpt) algorithm for \textsc{p-\subprob} is an algorithm whose runtime is bounded by $f(|G|)|H|^{O(1)}$ where $f$ is a computable function. That is, we allow some possible superpolynomial effort in $G$, as long as this yields an algorithm that is polynomial in the usually much bigger host graph $H$. 

In general, it is easy to see that  \psubprob is hard for the complexity class $\sf \#W[1]$ (counting the number of $k$-cliques in a graph is the canonical $\sf \#W[1]$-complete problem when parameterized by $k$). This class can be understood as the parameterized counting complexity equivalent of $\sf NP$ and it is a standard assumption in parameterized complexity that hardness for $\sf \#W[1]$ implies that there is no fpt algorithm for the problem. This means that there is no general algorithm that can efficiently compute $\subs$. However, recent results of \citet{CurticapeanDM17} have greatly clarified the picture for which specific pattern graphs $G$ an fpt algorithm is possible.
Namely, the complexity of \psubprob depends precisely on the treewidth of the graphs in $\spasm(G)$. 

In particular, there is an algorithm that runs in time $f(|G|) O(|V(H)|^{k+1})$ where $k$ is the maximal treewidth of graphs in the $\spasm$. The algorithm is conceptually simple once we know that $\subs(G, \cdot)$ is a graph motif parameter. 
\begin{enumerate}
    \item First compute $\spasm(G)$ and the corresponding coefficients. This clearly is independent of the graph $H$ and takes $g(|G|)$ time. This can be seen as a preprocessing step from a practical point of view.
    \item For each $F \in \spasm(G)$, compute $\homs(F,H)$ and then compute $\subs(G,\cdot)$ according to \Cref{lgmp}. Computing the number of homomorphisms $\homs(F,H)$ using a tree decomposition is well understood and requires $O(|H|^k)$ time, where $k$ is the treewidth of $F$ \citep{DiazST02}.
\end{enumerate}
Note that this algorithm is universal for any graph motif parameter. The only part that changes, is how to compute the basis and the coefficients in the first step. This of course entirely depends on the kind of graph motif parameter, but by definition of a graph motif parameter, the basis does not depend on the input graph ($H$ in this subgraph counting setting).

However, not only are homomorphism counts an efficient way to compute subgraph counts. In a sense, there can be no more efficient way to compute subgraph counts. Formally speaking, both sides of \Cref{lgmp} are fpt-interreducible to each other. More precise analysis reveals that these reductions can be performed very efficiently, and in fact so efficiently that we obtain very precise lower bounds for these problems. %
\begin{theorem}[Simplified, \citet{CurticapeanDM17}]
\label{complex.monotone}
    Let $G$ be a graph and $k$ be the maximum treewidth in $\spasm(G)$. Then there is no $f(|G|)|H|^{o(k/\log k)}$ time algorithm that computes $\subs(G, \cdot)$ for input $H$ if \#\textsf{ETH} holds.\footnote{\#\textsf{ETH}  is the counting version of the Exponential Time Hypothesis, a standard assumption in parameterized complexity.}
\end{theorem}

Note that it is conjectured that the bound of $o(k/\log k)$  can be improved to $o(k)$ \citep{Marx10}. Under this conjecture, this shows that there is no way to really significantly improve on the computation of $\subs(G,H)$ via the homomorphism basis as was described above.

Intuitively, this then means that graph motif parameters cannot be computed more efficiently than their basis terms.
Thus supporting the claim from \Cref{sec:efficient}, that computing a graph motif parameter directly cannot be significantly less computational effort than computing the homomorphism counts for the basis graphs.

\paragraph{Beyond $\subprob$}
First, we emphasize that the original result for \Cref{complex.monotone} by \citet{CurticapeanDM17} is not restricted to subgraph counting but holds for arbitrary graph motif parameters (the parameter then becomes the encoding size of the basis).
The same principle generalizes even further and is notably not limited to plain graphs. When we have directed graphs (and thus also directed patterns), the picture changes slightly but still no improvement over counting homomorphisms in the basis is possible \citep{BressanLR23}. Analogous statements hold with node and/or edge-labels, and general relational structures \citep{ChenM16}.

\paragraph{Practical Considerations}

In the context of this paper it is of particular note that the fpt algorithm that is obtained from using the homomorphism basis is practical. To the best of our knowledge, there is no general implementation for \subprob that does not use the homomorphism basis and avoids exponential runtime in either $|V(G)|$ or $|E(G)|$. The reason for this exponential runtime is that they enumerate through the possible subgraphs, solving the counting problem by actually counting the subgraphs at enormous cost.

In contrast, the fpt algorithm consists of two steps, both of which are feasible for realistic pattern sizes. First, the \spasm and the coefficients need to be computed. Even naive algorithms that enumerate all possible partitions are efficient enough up to patterns of size 12 without specialized hardware (from preliminary experiments).
In a second step, the algorithm requires the computation of $\Hom(F,H)$ for each $F \in \spasm(G)$. While this may superficially seem similar to \subprob, counting homomorphisms is much easier as they can be counted via dynamic programming on tree decompositions. This process requires time $O(|V(H)|^{k+1})$ where $k$ is the treewidth of $F$. In practice this is magnitudes easier on most patterns, e.g., compare the exponents mentioned above for the $7$-cycle $C_7$ to the observation that every graph in the basis of $C_7$ has treewidth $2$.

Of further practical relevance is that not only is counting homomorphisms easier than naive subgraph counting algorithms, but also much more flexible for large tasks. Recall that we need to compute $\homs(F,H)$ for each $F \in \spasm(G)$. However, these computations are entirely independent and can all be performed in parallel without any deeper algorithmic considerations. We demonstrate the efficacy of counting homomorphisms for GNN tasks by computing all the homomorphism counts for our experiments in significantly less time than is required for the training of the respective models. See \Cref{tab:counting-time} for details.

\section{More on Node-Level Counting}
\label{app:nodelevel}
In \Cref{sec:node.level} we show how homomorphism counts can also determine node level subgraph counts for each orbit of the pattern graph. We note that the \spasm itself is not enough to determine such node-level subgraph counts via \Cref{lgmp}.

We demonstrate this claim by example. In \Cref{ex:nodehoms} we show why node-level homomorphism for the part of $\spasm^\anchor$ without the second version of the 3 vertex path are insufficient. \Cref{ex:nodehoms1} and \Cref{ex:nodehoms2} illustrate the homomorphisms to two graphs $H_1$ and $H_2$ (gray arrows show mappings for a second homomorphism, all cases have only 1 or two homomorphisms), both clearly without any $\drawCfour$ subgraphs. 
From left to right there are $2$, $2$, and $1$ homomorphisms into $H_1$ and $2$, $1$, and $1$ homomorphisms into $H_2$.
Thus, to express the subgraph count as a linear combination of homomorphism counts would require constants $\alpha_1,\alpha_2,\alpha_3$ such that
\[
    \alpha_1 \cdot 2 + \alpha_2 \cdot 2 + \alpha_3 \cdot 1 = \alpha_1 \cdot 2 + \alpha_2 \cdot 1 + \alpha_3 \cdot 1
\]
Clearly this implies $\alpha_2=0$ and we see that these graphs are insufficient to express the number of 4-cycles a vertex is part of.
\begin{figure}[t]
\centering
\begin{subfigure}{0.87\textwidth}
        \centering
    \includegraphics[width=0.9\textwidth]{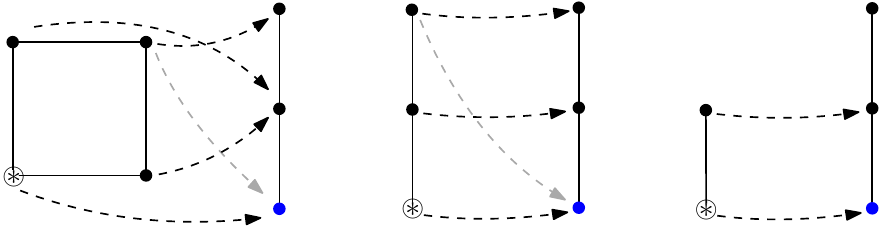}
    \caption{Homomorphisms from $\spasm(\drawCfour)$ to $H_1$ with the anchor mapping to the blue vertex.}
    \label{ex:nodehoms1}
\end{subfigure}

\begin{subfigure}{0.87\textwidth}
        \centering
    \includegraphics[width=0.9\textwidth]{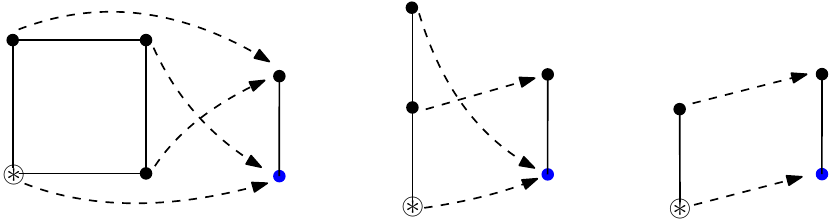}
    \caption{Homomorphisms from $\spasm(\drawCfour)$ to $H_2$ with the anchor mapping to the blue vertex.}
    \label{ex:nodehoms2}
\end{subfigure}
\caption{}
\label{ex:nodehoms}
\end{figure}

Similarly, it is insufficient to only count the paths from the 3-vertex path anchored in the middle vertex, or even being provided only the sum of both 3-vertex path counts. The counterexample for both these cases is the $n$-vertex star graph: the path anchored at the middle vertex has $(n-1)^2$ homomorphisms to the center vertex, whereas the number of homomorphisms from all three other patterns is $n-1$. 

To supplement the discussion of \Cref{app:complexity} we note that it is straightforward to reduce graph-level subgraph counting to node-level subgraph counting. Suppose we want to count $\subs(G,H)$. Create the anchored graph $G'$ by adding a new vertex $v'$ to $G$ that is connected to every other vertex. Additionally, set $v'$ as the anchor of $G'$. Similarly, create $H'$ by adding a new vertex $u'$ to $H$ that is connected to all other vertices.
It is then easy to see that $\subs(G,H) = \subs(G',H', u')$. Furthermore, we can observe that since new vertex $v'$ is connected to every other vertex, it can never be contacted with any other vertex to create a loop-free quotient. That is, every graph $F \in \spasm(G')$ is a graph in $\spasm(G)$ with an additional vertex that is connected to all other vertices. This means that the maximal treewidth in $\spasm(G')$ is at most $1$ higher than $\spasm(G)$.  Hence, we have a linear reduction from graph-level subgraph counting reduces to node-level subgraph counting, that increases the treewidth of the pattern at most by $1$, i.e., node-level subgraph counting also critically depends on the treewidth of the \spasm.

\section{Additional Experimental Details}
\label{app:exp-details}

Here, we provide all the experimental details and hyperparameters for the results presented in Section \ref{sec:experiments}. The code and instructions to reproduce our results can be found at this GitHub repository: \href{https://github.com/ejin700/hombasis-gnn}{https://github.com/ejin700/hombasis-gnn}.

\paragraph{Compute Resources.} 
All homomorphism counting and experiments for ZINC, COLLAB, and BREC were run on a cluster with NVIDIA A10 GPUs (24 GB). Each node had 64 cores of Intel(R) Xeon(R) Gold 6326 CPU at 2.90GHz and ~500GB of RAM. QM9 experiments were run on a cluster with NVIDIA A100 GPUs. All experiments used 1 GPU at a time.

\subsection{Runtime for Counting Homomorphisms}
\label{app:runtimes}

\begin{table*}[t]
\caption{Time to compute all the homomorphism counts used in each experiment in Section \ref{sec:experiments} reported as elapsed real time (wall time). }
\label{tab:counting-time}
\begin{sc}
\begin{small}
\begin{center}
\begin{tabular}{lccccccc}
\toprule

Dataset                 & \#Graphs                 & avg. \#Nodes                     & avg. \#Edges                     & Basis & Basis size & \#cores & Time   \\
\midrule
\multirow{3}{*}{ZINC}   & \multirow{3}{*}{12,000}  & \multirow{3}{*}{$\sim$23.2} & \multirow{3}{*}{$\sim$49.8} & $\spasm(C_7)$    & 12         & 48      & 1m24s      \\
                        &                          &                             &                             & $\spasm(C_8)$    & 35         & 48      & 4m45s      \\
                        &                          &                             &                             & $\spasm^{\anchor}(C_7, C_8)$   & 118        & 48      & 19m10s \\
\midrule
\multirow{2}{*}{QM9}    & \multirow{2}{*}{130,831} & \multirow{2}{*}{$\sim$18.0} & \multirow{2}{*}{$\sim$37.3} & $\graphs^{\sf con}_{\leq 5}$    & 30         & 48      & 42m15s \\
                        &                          &                             &                             & $C_6$    & 1          & 48      & 2m2s      \\
\midrule
\multirow{6}{*}{COLLAB} & \multirow{6}{*}{1}       & \multirow{6}{*}{235,868}    & \multirow{6}{*}{967,632}  & $K_3$  & 1     & 20      & 3m48s      \\
                        &                          &                             &                             & $K_4$    & 1          & 20      & 5m30s      \\
                        &                          &                             &                             & $K_5$    & 1          & 20      & 6m29s      \\
                        &                          &                             &                             & $\spasm(P_4)$    & 4          & 20      & 4m16s  \\
                        &                          &                             &                             & $\spasm(P_5)$    & 8          & 20      & 39m23s \\
                        &                          &                             &                             & $\spasm(P_6)$    & 15         & 20      & 3h6m      \\
\midrule
BREC                    & 51,200                   & $\sim$34.7                 & $\sim$286.9                & $\graphs^{\sf con}_{\leq 5}$    & 30         & 48      & 53m53s      \\
\bottomrule
\end{tabular}
\end{center}
\end{small}
\end{sc}
\end{table*}

We report the time required to compute all the homomorphism counts that were used for each experiment in Table \ref{tab:counting-time}. We also provide additional details regarding the size of the graphs in each dataset and the amount of compute used to generate these counts. From Table \ref{tab:counting-time}, we can clearly see that the time required to generate these homomorphism counts is almost negligible, especially since this counting step need only be performed once for a given parameter and dataset. In other words, once the counts have been obtained, they can be added to any existing model or architecture. 

The computational efficiency of performing this homomorphism counting is especially apparent for QM9, where we are able to generate all $\homs(\graphs^{\sf con}_{\leq 5})$ counts for over 130,000 individual graphs in roughly 40 minutes. This averages out to be less than 0.02 seconds per graph to count the homomorphisms for all 5-vertex connected components. For comparison, training a single run of R-GCN on QM9 to predict a single target property took several hours (more details on this in Appendix \ref{app:qm9}).

For ogbl-COLLAB, we only perform counting on the training graph, which is a subset of the entire dataset. This is because for link prediction, a subset of the edges are removed from training and used for validation and testing. Due to the memory constraints of our compute resources, we use a slightly modified counting method for counting the number of 5-clique homomorphisms in COLLAB, which we provide in our codebase. Finally, we note that we do not perform counting for $\homs(K_5)$ on the 4-vertex graphs in BREC because the graphs are constructed in a way that makes counting them infeasible.

\subsection{Generating Node-level Features using Homomorphism Counts}

In multiple of the experiments we inject homomorphism counts of graphs in $\spasm(G)$ at node-level. In the terminology of \Cref{sec:node.level}, for every graph $F \in \spasm(G)$, we pick an arbitrary vertex as its anchor to create $F'$. For each vertex $v$ of input graph $H$ we then compute the vector
\[
\mathbf{f}_v := (\homs(F',H)[\anchor \to v])_{F \in \spasm(G)}
\]
for some fixed order on the graphs of $\spasm(G)$. We then associate each $\mathbf{f}_v$ to node $v$ as additional features. The details in which this vector is added depend on the experiment and is described in the respective experiment section.

The motivation for using this approach is that it splits the number of total homomorphisms up ``disjointly", i.e., no homomorphism is counted twice at different vertices. In particular this means that
\begin{equation}
\label{homsum}
\sum_{v \in V(H)} \mathbf{f}_v = (\homs(F,H))_{F \in \spasm(G)}.
\end{equation}
That is, the global counts are simply a sum over the individual feature vectors, without the need to correct for potential double counting of homomorphisms. This thus presents a compromise of preserving the global counts as truthfully as possible, while still providing some node-level information. Although as discussed in \Cref{sec:node.level} and \Cref{app:nodelevel}, precise node-level information requires more fine-grained homomorphism information.

\begin{table*}
\caption{Hyperparameters for ZINC experiments.}
\label{tab:zinc-hp}
\vskip 0.15in
\begin{center}
\begin{small}
\begin{sc}
\begin{tabular}{lcccccc}
\toprule
model                & layers & hidden dim & batch size & max epochs & \#heads & readout \\
\midrule
\multirow{2}{*}{GAT} & 4      & 18       & 128        & 1000       & 8       & mean    \\
                     & 16     & 22       & 128        & 1000       & 8       & mean    \\
\midrule
\multirow{2}{*}{GCN} & 4      & 125      & 128        & 1000       & N/A     & mean    \\
                     & 16     & 172      & 128        & 1000       & N/A     & mean    \\
\midrule
\multirow{2}{*}{GIN} & 4      & 110      & 128        & 1000       & N/A     & sum     \\
                     & 16     & 124      & 128        & 1000       & N/A     & sum     \\
\midrule
BasePlanE            & 3      & 128      & 256        & 500        & N/A     & sum     \\        
\bottomrule
\end{tabular}
\end{sc}
\end{small}
\end{center}
\vskip -0.1in
\end{table*}

\subsection{ZINC Experimental Details}
\label{app:zinc}
For the graph regression task on the ZINC dataset, the objective is to predict the constrained solubility of molecules. We use the same hyperparameters and evaluation protocol as in \citet{dwivedi2023benchmarking}. As such, we use a subset of the ZINC dataset that contains 12,000 individual graphs, which are split into 10,000 graphs for training, 1,000 graphs for validation, and 1,000 graphs for testing. For each graph, the original vertex features denote the type of heavy atom for a given node, and the edge features denote the type of bond between two nodes.  \citet{dwivedi2023benchmarking} do not use edge features when performing model benchmarking, so for comparability, we adhere to their protocol and also omit them from our experiments in Section \ref{sec:gr-zinc}. All results are reported as the average of 4 runs at different seeds.

\subsubsection{Model selection and training}
For ZINC, we select GAT \cite{VelickovicCCRLB18}, GCN \cite{Kipf16}, GIN \cite{xu18}, and BasePlanE \cite{Dimitrov2023} as our baseline models. This is because GAT, GCN, and GIN are all standard architectures, and we want to evaluate how effective the homomorphism basis counts are at elevating the performance and expressivity of these models. We select BasePlanE because it is a recent strong architecture that is maximally expressive over planar graphs (all graphs in ZINC are planar), and it has been show to perform very well on several molecular benchmarks. Both GIN and BasePlanE theoretically reach 1-WL expressiveness, so we are able to demonstrate the fine-grained expressiveness of our approach by adding only the higher-order basis terms to the model. Note that our reported results for GIN are slightly stronger than those presented in the original benchmark because we use a more standard implementation of GIN from the PyTorch Geometric library. 

In terms of training, we use the same hyperparameters for GAT, GCN, and GIN as reported by \citet{dwivedi2023benchmarking}, and the same hyperparameters for BasePlanE as reported by \citet{Dimitrov2023}. We also list these hyperparameters in \Cref{tab:zinc-hp}. Following the setup from \citet{dwivedi2023benchmarking}, we test GAT, GCN, and GIN with both 4 and 16 layers. The full results for ZINC are in Table \ref{tab:zinc-full}. For GAT, GCN, and GIN, the training is done using the following configuration: $optimizer=Adam$, $initial\_lr = 0.001$, $lr\_reduce\_factor=0.5$, $minimum\_lr=1e-5$, $patience=10$. For BasePlanE, the training is done with same configuration as listed, but with a $patience=30$ instead of $10$.

Each model is evaluated using the mean average error (MAE), which is defined as:

\begin{equation}
    \text{MAE} = \frac{\sum_{i=1}^n |y_i - x_i|}{n}
\end{equation}

where $y_i$ is the predicted value, $x_i$ is the true value, and $n$ is the total number of data points. 

In addition to providing the full extended results in \Cref{tab:zinc-full}, we also provide a comparison of our BasePlanE results to other leading models on the ZINC12k dataset (without edge features) in \Cref{tab:bp-zinc-compare}. From this comparison, we can see that with the addition of the $\spasm$ and $\spasm^\anchor$ homomorphism counts for $C_3,...,C_8$, we are able to achieve extremely competitive results. This further shows that our method does not only work on weak baseline models, but it can also boost the performance of more complex GNN architectures.

\subsubsection{Selecting the graph motif parameters}
We select cycles as our patterns of interest because it is well established that rings are an important substructure in molecules. We use cycles up to length 8 for our experiments in Table \ref{tab:basis-v-homiso.column} because \citet{BouritsasFZB23} found that it yielded the best performance for their GSN-v model on the ZINC dataset. Note that we take the counts for $\spasm(C_7) \cup \spasm(C_8)$ since this contains all the component elements that make up the basis for $C_3 - C_8$. Additionally, as mentioned in Section \ref{sec:efficient}, we can use these $\spasm$ and $\spasm^\anchor$ counts to quickly calculate graph-level or node-level subgraph counts, respectively. Therefore, we also include these subgraph counts as additional features since this information directly follows from computing the homomorphism counts of the entire basis of interest. In each model, the count features are encoded using a 2-layer MLP and concatenated to the initial node feature embeddings. 

This 2-layer MLP that we use to encode the count features differs from the approach taken by \citet{BarceloGRR21}, who use the z-score of the log-normalized count values. However, we found that this normalization method actually degraded the performance of using $\subs$ counts. Therefore, we decided to pass the additional count features into an MLP, which would yield more fair and comparable results.

\begin{table*}
\caption{Extended MAE results on ZINC12k graph regression (without edge features) with $C_3,...,C_8$.}
\label{tab:zinc-full}
\vskip 0.15in
\begin{center}
\begin{small}
\begin{sc}
\begin{tabular}{lcccccc}
\toprule
Model & Layers & Base & $\subs$ & $\homs$ & $\spasm$  & $\spasm^{\anchor}$  \\
\midrule
\multirow{2}{*}{GAT} & 4      & 0.457$\pm 0.004$    & 0.210$\pm 0.006$  & 0.269$\pm 0.033$ & 0.155$\pm 0.006$   & 0.147$\pm 0.004$ \\
                     & 16     & 0.380$\pm 0.009$    & 0.201$\pm 0.004$  & 0.229$\pm 0.008$ & 0.155$\pm 0.008$   & 0.152$\pm 0.007$ \\
\multirow{2}{*}{GCN} & 4      & 0.417$\pm 0.007$    & 0.206$\pm 0.006$  & 0.254$\pm 0.017$ & 0.166$\pm 0.003$   & 0.165$\pm 0.004$ \\
                     & 16     & 0.282$\pm 0.007$    & 0.198$\pm 0.003$  & 0.235$\pm 0.005$ & 0.167$\pm 0.007$   & 0.166$\pm 0.005$ \\
\multirow{2}{*}{GIN} & 4      & 0.294$\pm 0.012$    & 0.147$\pm 0.006$  & 0.208$\pm 0.025$ & 0.158$\pm 0.004$   & 0.146$\pm 0.005$ \\
                     & 16     & 0.246$\pm 0.019$    & 0.143$\pm 0.002$  & 0.197$\pm 0.015$ & 0.157$\pm 0.005$   & 0.143$\pm 0.004$ \\
BasePlanE            & 3      & 0.124$\pm 0.004$    & 0.108$\pm 0.002$  & 0.106$\pm 0.004$ & 0.104$\pm 0.005$   & 0.100$\pm 0.002$ \\
\bottomrule
\end{tabular}
\end{sc}
\end{small}
\end{center}
\vskip -0.1in
\end{table*}

\begin{table*}
\caption{Comparison of BasePlanE results with other leading models on the ZINC12k dataset (without edge features).}
\label{tab:bp-zinc-compare}
\vskip 0.15in
\begin{center}
\begin{small}
\begin{sc}
\begin{tabular}{lc}
\toprule
model                & MAE \\
\midrule
GCN         & $0.278 \pm \tiny0.003$ \\
GIN(-E)     & $0.387 \pm \tiny0.015$ \\
PNA         & $0.320 \pm \tiny0.032$ \\
GSN         & $0.140 \pm \tiny0.006$ \\
CIN         & $0.115 \pm \tiny0.003$ \\
BasePlanE 	& $0.124 \pm \tiny0.004$ \\
BasePlanE + $\subs$ 	& $0.108 \pm \tiny0.002$ \\
BasePlanE + $\homs$ 	& $0.106 \pm \tiny0.004$ \\
BasePlanE + $\spasm$ 	& $0.104 \pm \tiny0.005$ \\
BasePlanE + $\spasm^{\anchor}$ 	& $0.100 \pm \tiny0.002$ \\

\bottomrule
\end{tabular}
\end{sc}
\end{small}
\end{center}
\vskip -0.1in
\end{table*}

\subsection{QM9 Experimental Details}
\label{app:qm9}

For the graph regression task on the QM9 dataset, the objective is to predict 13 different target properties for the set of molecules \cite{wu2018moleculenet}. The entire dataset contains 130,831 graphs, which are split into 110,831 for training, 10,000 for validation, and 10,000 for testing. The original node features denote atom type as well as other chemical information about the atom. Similarly to ZINC, the edge features denote the type of bond between two nodes. 

We select R-GCN \cite{schlichtkrull2018modeling} as our base model, report the baseline results from \citet{Brockschmidt20}, and adapted the codebase from \citet{AbboudDC22} to perform our experiments. Our implementation of R-GCN with the addition of the homomorphism counts from $ \graphs^{\sf con}_{\leq 5} \cup C_6$ uses the following hyperparameters: $batch\_size=128$, $hidden\_dim=128$, $layers=8$, $lr=0.001$, $aggregation=mean$, $epochs=300$, and $optimizer=Adam$. We also use both batch and layer normalization and include residual connections at a frequency of every 2 layers. Our results in Table \ref{tab:qm9.smol} are reported as an average across 5 different runs. On average, the model training time for R-GCN on one property took roughly 5 hours to complete, which is significantly longer than the time it took to compute all of the homomorphism feature counts for the entire QM9 dataset (see \Cref{tab:counting-time}). This further highlights the negligible compute cost to generate the homomorphism counts.

\citet{Alon-ICLR21} presented a slightly modified version of R-GCN that uses a fully-adjacent (FA) layer at the very end of the model. They show that this helps counteract the signal loss due to over-squashing effects. We also test with this model configuration and find that using the homomorphism basis counts continue to improve the performance of the R-GCN+FA model (see Table \ref{tab:qm9}). For this experiment, we align with the R-GCN+FA hyperparameter setup which uses: $batch\_size=128$, $hidden\_dim=128$, $layers=8$, $lr=0.000572$, $aggregation=sum$, $epochs=400$, and $optimizer=RMSProps$. We use batch and layer normalization, as well as residual connections every 2 layers.

We notice that for all but one property, the R-GCN + $\homs(\graphs^{\sf con}_{\leq 5}, C_6)$ model outperforms the R-GCN + FA model. This suggests that for tasks where the molecular structure is important, the graph-level homomorphism counts are still enough to provide the model with useful information, despite potential signal loss from oversquashing. Therefore, our method can be seen as complementary to using a final FA layer to help further boost the performance of an originally weak model. 

\subsubsection{Selecting the graph motif parameters}
Because there are multiple different target tasks for the QM9 dataset, we follow \Cref{prop:indsubs} and decide to take the homomorphism counts for all connected components that contain up to 5 vertices, $\homs(\graphs^{\sf con}_{\leq 5})$. In addition to this set, we also include $\homs(C_6)$ because it is well established that 6-cycles are prevalent within molecules. Also, this ensures that $\spasm(C_6)$ is contained within this set of features. Since we use R-GCN as our baseline model, we include all homomorphism counts for all components in this set, including the 1-WL ones. For each graph, all counts are encoded as graph-level features. Specifically, we concatenate the raw homomorphism counts at the very end after the message passing layers to the final feature representation vector for the graph before the final MLP decoder.

\begin{table}[t]
\caption{We report the mean absolute error (MAE) for graph regression on QM9. Using the homomorphism counts yields significant improvements over the baseline models (both with and without FA).}
\label{tab:qm9}
\vskip 0.15in
\begin{center}
\begin{small}
\begin{sc}
\begin{tabular}{lrrrr} 
\toprule
Property & R-GCN & + $\homs(\graphs^{\sf con}_{\leq 5} \cup C_6)$ & \multicolumn{1}{c}{+ FA} & \multicolumn{1}{c}{ + FA/\homs} \\ 
\midrule
mu       & 3.21 \vartxtm{0.06}  & 2.29 \vartxtm{0.03}   & 2.91 \vartxtm{0.07}   & 2.15 \vartxtm{0.02}   \\
alpha    & 4.22 \vartxtm{0.45}  & 1.77 \vartxtm{0.05}   & 2.14 \vartxtm{0.08}   & 1.70 \vartxtm{0.05}   \\
HOMO     & 1.45 \vartxtm{0.01}  & 1.30 \vartxtm{0.03}   & 1.37 \vartxtm{1.41}   & 1.22 \vartxtm{0.02}   \\
LUMO     & 1.62 \vartxtm{0.04}  & 1.41 \vartxtm{0.02}   & 1.41 \vartxtm{0.01}   & 1.27 \vartxtm{0.00}   \\
gap      & 2.42 \vartxtm{0.14}  & 2.00 \vartxtm{0.04}   & 2.03 \vartxtm{0.03}   & 1.81 \vartxtm{0.03}   \\
R2       & 16.38 \vartxtm{0.49} & 10.29 \vartxtm{0.35}  & 13.55 \vartxtm{0.50}  & 9.99 \vartxtm{0.33}   \\
ZPVE     & 17.40 \vartxtm{3.56} & 3.03 \vartxtm{0.38}   & 5.81 \vartxtm{0.61}   & 2.86 \vartxtm{0.31}   \\
U0       & 7.82 \vartxtm{0.80}  & 1.09 \vartxtm{0.18}   & 1.75 \vartxtm{0.18}   & 1.03 \vartxtm{0.05}   \\
U        & 8.24 \vartxtm{1.25}  & 1.21 \vartxtm{0.17}   & 1.88 \vartxtm{0.22}   & 1.07 \vartxtm{0.04}   \\
H        & 9.05 \vartxtm{1.21}  & 1.22 \vartxtm{0.14}   & 1.85 \vartxtm{0.18}   & 1.14 \vartxtm{0.12}   \\
G        & 7.00 \vartxtm{1.51}  & 1.14 \vartxtm{0.13}   & 1.76 \vartxtm{0.15}   & 1.19 \vartxtm{0.18}   \\
Cv       & 3.93 \vartxtm{0.48}  & 1.46 \vartxtm{0.08}   & 1.90 \vartxtm{0.07}   & 1.46 \vartxtm{0.07}   \\
Omega    & 1.02 \vartxtm{0.05}  & 0.81 \vartxtm{0.02}   & 0.75 \vartxtm{0.04}   & 0.71 \vartxtm{0.02}   \\
\bottomrule
\end{tabular}
\end{sc}
\end{small}
\end{center}
\vskip -0.1in
\end{table}

\subsection{COLLAB Experimental Details}
\label{app:collab}

COLLAB is a collaboration network graph from Open Graph Benchmark that presents a link prediction task \cite{OGB-NeurIPS2020}. In total, the entire graph contains 1,285,465 edges, but 23,367 of them are used for validation, and 294,466 edges are used for testing. Nodes represent scientists, and edges represent collaborations between two scientists. The nodes also contain 128-dimensional node features that correspond to the average word embeddings for each scientist's papers. Edge features include the year and the edge weight, representing the number of co-authored papers published that year. 

\subsubsection{Model selection and training}
We select GAT, GCN, and GraphSAGE \cite{HamiltonYL17} as our default baseline models. For GCN and GraphSAGE, we use the exact same model and hyperparameters as the default example provided by \citet{OGB-NeurIPS2020}. For GAT, we construct a model that closely aligns with the overall framework of the GCN and GraphSAGE implementations. These model details are outlined in Table \ref{tab:collab-hp}, and our full results are presented in Table \ref{tab:collab-paths.appendix}. All results are reported as an average of 4 runs at different seeds.
 
\subsubsection{Evaluation}
Models are evaluated using the Hits@50 metric, which is a rank-based metric. Each test triplet $t_i = (v_{head},r,v_{tail}) \in T_{test}$, is ranked against a set of corrupted facts that are generated by modifying the head and tail entity of the triplet of interest. Using these corrupted triplets, the model assigns a score to each triplet and sorts them in descending order by their predicted scores. The final $rank(t_i)$ of the ``true" test triplet is computed as its index in the resulting sorted list. Repeating this process for all triplets in the test set $T_{test}$ will result in a set of individual rank scores $Q = \{ rank(t_i), ..., rank(t_n) \}$, that is used to compute the overall performance metric.

The Hits@k metric calculates how many true facts are ranked in the top $k$ positions against their respective negatives, and divides that value by the number of triplets in the test set. For example, if a test set contains two true facts, and one of them ranks first amongst its negatives, while the other one ranks second amongst its negatives, the Hits@1 value would equal $1/2 = 0.5$, whereas the Hits@3 value would equal $2/2 = 1$. The formal definition of Hits@k is given in Equation \ref{eq:hits}:

\begin{equation}
    \label{eq:hits}
    Hits@k := \frac{| \{ {rank(t_i) \in Q \mid rank(t_i) \leq k \}}|}{|Q|}
\end{equation}

\subsubsection{Selecting the graph motif parameters}
Traditional substructure-style approaches \cite{BarceloGRR21} have used cliques as patterns of interest for the COLLAB dataset. However, because COLLAB presents a link prediction task that may require multi-hop reasoning, paths could also be a useful pattern. Paths are interesting because counting the number of $n$-vertex paths in a graph cannot be done by a 1-WL model due to the cyclic components in the homomorphism basis. Also, standard subgraph counting techniques \cite{RibeiroPSAS21} are mostly intractable for graphs as large as COLLAB, whereas our method is practically feasible.

Due to the large graph size, we normalize our raw homomorphism counts using the sine and cosine positional encoding technique from \citet{vaswani2017attention}. We experiment with positional encoding dimensions of size 4, 6, and 8. Specific hyperparameter details are presented in Table \ref{tab:collab-hp}.

\begin{table*}[t]
\caption{Hyperparameters for COLLAB experiments.}
\label{tab:collab-hp}
\vskip 0.15in
\begin{center}
\begin{small}
\begin{sc}
\begin{tabular}{lccccccc}
\toprule
 model&  layers&  hidden dim& batch size&  learning rate&  epochs&  pe dim  & \#heads\\
\midrule
 GAT&  3&  256&  65536&  0.001&  400&  8&   8\\
 GCN&  3&  256&  65536&  0.001&  400&  8&   N/A\\
 GraphSAGE&  3&  256&  65536&  0.001&  400&  6& N/A\\
\bottomrule
\end{tabular}
\end{sc}
\end{small}
\end{center}
\vskip -0.1in
\end{table*}

\begin{table}[t]
\caption{Extended Hits@50 results for link prediction over the COLLAB dataset.}
\label{tab:collab-paths.appendix}
\begin{center}
\begin{small}
\begin{sc}
\begin{tabular}{clccc}
\toprule
add. basis & $\Gamma$ & GCN & GraphSAGE & GAT \\
\midrule
& {\scriptsize ---}       & 46.13\%\tiny{$\pm 2.10$} & 48.85\%\tiny{$\pm 0.43$} & 48.27\%\vartxt{1.05} \\
\midrule

& $K_3$           & 49.41\%\tiny{$\pm 0.42$} & 49.55\%\tiny{$\pm 0.66$} & 49.43\%\vartxt{0.73} \\

& $K_4$           & 47.76\%\tiny{$\pm 0.53$} & 49.29\%\tiny{$\pm 0.30$} & 48.35\%\vartxt{0.85}\\

& $K_5$           & 48.01\%\tiny{$\pm 0.87$} & 49.61\%\tiny{$\pm 0.24$} & 49.75\%\vartxt{0.16}\\

\midrule
\drawPone \drawPtwo \!\! \drawKthree& $P_4$   & 49.59\%\tiny{$\pm 0.23$} & 50.01\%\tiny{$\pm 0.57$} & 50.76\%\vartxt{0.51}\\

\drawStarthree \drawKthreeplusone \drawCfour & $P_5$   & 49.60\%\tiny{$\pm 0.29$} & 49.50\%\tiny{$\pm 0.47$} & 51.55\%\vartxt{0.94} \\

 \drawCfourPlusOne\,  \drawAlmostKfour \drawKthreePlusTwo \drawKthreeplusPtwo \drawStarthreePlusOne  \drawCfive  & $P_6$   & 50.35\%\tiny{$\pm 0.21$} & 50.18\%\tiny{$\pm 0.14$} & { 51.62\%}\vartxt{0.66}\\
\bottomrule
\end{tabular}
\end{sc}
\end{small}
\end{center}
\vskip -0.1in
\end{table}

\subsection{BREC Experimental Details}
\label{app:brec}

BREC is a new synthetic expressiveness dataset, where the task is to distinguish 400 pairs of non-isomorphic graphs \cite{wang2023towards}. The graphs span four different primary categories: basic, regular, extension, and CFI. There are 60 pairs of basic graphs, 140 pairs of regular graphs, 100 pairs of extension graphs, and 100 pairs of CFI graphs. The regular graphs are further subdivided into 50 pairs of simple regular graphs, 50 pairs of strongly regular graphs, 20 pairs of 4-vertex condition graphs, and 20 pairs of distance regular graphs. None of the graphs in this dataset contain any initial node or edge features.

We present our results on BREC in the full context of results from \citet{wang2023towards} in \Cref{tab:n-vertex.big}. We note that the $2$-WL row in \citet{wang2023towards} is labeled as ``$3$-WL" in the original due to a mismatch in which style of $k$-WL is considered in the paper. Since we use $k$-WL to refer to the folklore version of $k$-WL, we have adapted the label from their table accordingly.

\begin{table*}[t]
\caption{Full results for the BREC expressiveness experiment, comparing the performance of all models tested by \citet{wang2023towards}.}
\label{tab:n-vertex.big}
\vskip 0.15in
\begin{center}
\begin{small}
\begin{sc}
\begin{tabular}{lcccccccccc}
\toprule
& \multicolumn{2}{c}{Basic (60)} & \multicolumn{2}{c}{Regular (140)} & \multicolumn{2}{c}{Extension (100)} & \multicolumn{2}{c}{CFI (100)} & \multicolumn{2}{c}{Total (400)} \\
\cmidrule{2-3}
\cmidrule{4-5}
\cmidrule{6-7}
\cmidrule{8-9}
\cmidrule{10-11}
Model & Num & Acc & Num & Acc & Num & Acc & Num & Acc & Num & Acc \\
\midrule

2-WL    &   60  &   100\%   &   50  &   35.7\%  &   100 &   100\%   &   60  &   60\%    &   270 &   67.5\%  \\
SPD-WL  &   16  &   26.7\%  &   14  &   11.7\%  &   41  &   41\%    &   12  &   12\%    &   83  &   20.8\%  \\
$S_3$   &   52  &   86.7\%  &   48  &   34.3\%  &   5   &   5\%     &   0   &   0\%     &   105 &   26.2\%  \\
$S_4$   &   60  &   100\%   &   99  &   70.7\%  &   84  &   84\%    &   0   &   0\%     &   243 &   60.8\%  \\
$N_1$   &   60  &   100\%   &   99  &   85\%    &   93  &   93\%    &   0   &   0\%     &   252 &   63\%    \\
$N_2$   &   60  &   100\%   &   138 &   98.6\%  &   100 &   100\%   &   0   &   0\%     &   298 &   74.5\%  \\
$M_1$   &   60  &   100\%   &   50  &   35.7\%  &   100 &   100\%   &   41  &   41\%    &   251 &   62.8\%  \\
\midrule

NGNN    &   59  &   98.3\%  &   48  &   34.3\%  &   59  &   59\%    &   0   &   0\%     &   166 &   41.5\%  \\
DE+NGNN &  60  &   100\%   &   50  &   35.7\%  &   100 &   100\%   &   21  &   21\%    &   231 &   57.8\%   \\
DS-GNN  &   58  &   96.7\%  &   48  &   34.3\%  &   100 &   100\%   &   16  &   16\%    &   222 &   55.5\%  \\
DSS-GNN &  58  &   96.7\%  &   48  &   34.3\%  &   100 &   100\%   &   15  &   15\%    &   221 &   55.2\%   \\
SUN     &   60  &   100\%   &   50  &   35.7\%  &   100 &   100\%   &   13  &   13\%    &   223 &   55.8\%  \\
SSWL\_P  &   60  &   100\%   &   50  &   35.7\%  &   100 &   100\%   &   38  &   38\%    &   248 &   62\%   \\
GNN-AK  &   60  &   100\%   &   50  &   35.7\%  &   97  &   97\%    &   15  &   15\%    &   222 &   55.5\%  \\
KP-GNN  &   60  &   100\%   &   106 &   75.7\%  &   98  &   98\%    &   11  &   11\%    &   275 &   68.8\%  \\
$\text{I}^2$-GNN& 60 & 100\%&   100 &   71.4\%  &   100 &   100\%   &   21  &   21\%    &   281 &   70.2\%  \\
PPGN   &   60  &    100\%  &   50  &   35.7\%  &   100 &   100\%   &   23  &   23\%    &   233 &   58.2\%  \\
$\delta$-k-LGNN& 60 & 100\% &   50  &   35.7\%  &   100 &   100\%   &   6   &   6\%     &   216 &   54\%    \\
KC-SetGNN&  60  &   100\%   &   50  &   35.7\%  &   100 &   100\%   &   1   &   1\%     &   211 &   52.8\%  \\
GSN     &   60  &   100\%   &   99  &   70.7\%  &   95  &   95\%    &   0   &   0\%     &   254 &   63.5\%  \\
DropGNN &   52  &   86.7\%  &   41  &   29.3\%  &   82  &   82\%    &   2   &   2\%     &   177 &   44.2\%  \\
OSAN    &   56  &   93.3\%  &   8   &   5.7\%   &   79  &   79\%    &   5   &   5\%     &   148 &   37\%    \\
Graphormer&16  &   26.7\%  &   12  &   10\%    &   41  &   41\%    &   10  &   10\%    &   79  &   19.8\%  \\
\midrule

GIN + $\homs(\graphs_{\leq 5}^{\sf con})$ (Ours)&60&   100\%   &   120 &   85.7\%  &   96  &   96\%    &   0   &   0\%     &   276 &   69\%    \\
PPGN + $\homs(\graphs_{\leq 5}^{\sf con})$ (Ours) &60&  100\%   &   120 &   85.7\%  &   100 &   100\%   &   25  &   25\%    &   305 &   76.25\% \\

\bottomrule
\end{tabular}
\end{sc}
\end{small}
\end{center}
\vskip -0.1in
\end{table*} 
We follow the evaluation protocol of \citet{wang2023towards} and report the best results from 10 seeds to get an upper bound on the expressiveness of each model. The effect of this protocol on our results is minimal. For the top performing PPGN+$\homs(\graphs^{\sf con}_{\leq k})$ we only observe minimal variance over changing seeds. The number of solved CFI graph pairs varied between $22\%$ and $25\%$, and we observed no change in the other categories.

For GIN+$\homs(\graphs^{\sf con}_{\leq k})$, we use the exact same model construction and hyperparameters as was used to benchmark GSN on this dataset. Specifically, we used: $layers=4$, $lr=\text{1e-4}$, $weight\_decay=\text{1e-5}$, $batch\_size=16$, $output\_dimension=64$, $epochs=20$. For PPGN+$\homs(\graphs^{\sf con}_{\leq k})$, we slightly tweak the hyperparameters of the model used in the BREC benchmark by changing the dimensions of the blocks from 32 to 48 in order to accommodate our extra features. Further details, model configurations, and setups can be found in our \href{https://github.com/ejin700/hombasis-gnn}{github repository}.

\end{document}